%% file: Arxiv_submission.tex
\documentclass[11pt]{article}

\usepackage{siunitx}
\input{wg_style}

\newcommand*\samethanks[1][\value{footnote}]{\footnotemark[#1]}

\usepackage{authblk}

\title{\bf{\LARGE{Multi-Source Causal Inference Using Control Variates}}}

\author[1]{Wenshuo Guo\thanks{wguo@cs.berkeley.edu. Wenshuo Guo and Serena Wang contributed equally to this work.}}
\author[1]{Serena Wang\samethanks}
\author[2]{Peng Ding}
\author[1]{Yixin Wang}
\author[1,2]{Michael I. Jordan}
\affil[1]{Department of Electrical Engineering and Computer Sciences, University of California, Berkeley}
\affil[2]{Department of Statistics, University of California, Berkeley}

\begin{document}
\sloppy

\maketitle

\begin{abstract} 
\input{sec/abstract}
\end{abstract}

\input{sec/intro}
\input{sec/related_work}

\input{sec/prelim}
\input{sec/methods}

\input{sec/estimate_OR}
\input{sec/experiments}

\input{sec/conclu}

\section*{Acknowledgements}
This work was supported in part by the Mathematical Data Science program of the Office of Naval Research under grant number N00014-18-1-2764.

\bibliographystyle{plainnat}
\bibliography{ref}

\newpage\appendix
\input{sec/app_kernel}

\input{sec/app_bootstrap}
\input{sec/app_proofs}
\input{sec/app_exp_sim}

\input{sec/app_exp_real}

\end{document}

%% file: wg_style.tex
\usepackage{microtype}
\usepackage{graphicx}
\usepackage{subfigure}
\usepackage{booktabs} 
\usepackage{natbib}
\usepackage{epsfig,epsf,fancybox}
\usepackage{amsmath}
\usepackage{mathrsfs}
\usepackage{amssymb}
\usepackage{color}
\usepackage{multirow}
\usepackage{paralist}
\usepackage{verbatim}
\usepackage{algorithm}
\usepackage{algorithmic}
\usepackage{galois}
\usepackage{boxedminipage}
\usepackage{accents}
\usepackage{stmaryrd}
\usepackage[bottom]{footmisc}
\usepackage{tikz}

\usepackage{natbib}

\usepackage{url}
\usepackage[colorlinks,linkcolor=magenta,citecolor=blue, pagebackref=true]{hyperref}
\renewcommand*{\backrefalt}[4]{%
    \ifcase #1 \footnotesize{(Not cited.)}%
    \or        \footnotesize{(Cited on page~#2.)}%
    \else      \footnotesize{(Cited on pages~#2.)}%
    \fi}

\usepackage{amsthm} 

\makeatletter
\newtheorem*{rep@theorem}{\rep@title}
\newcommand{\newreptheorem}[2]{%
\newenvironment{rep#1}[1]{%
 \def\rep@title{#2 \ref{##1}}%
 \begin{rep@theorem}}%
 {\end{rep@theorem}}}
\makeatother

\newtheorem{theorem}{Theorem}[section]
\newreptheorem{theorem}{Theorem}

\newreptheorem{corollary}{Corollary}
\newtheorem{definition}{Definition}[section]

\newreptheorem{lemma}{Lemma}
\newtheorem{assumption}[theorem]{Assumption}

\newtheorem{proposition}[theorem]{Proposition}

\usepackage{hyperref}

\newcommand{\E}{\mathbb{E}}

\newcommand{\T}{\top}

\newcommand{\var}{\textnormal{Var}}
\newcommand{\cov}{\textnormal{Cov}}

\newcommand{\vw}{\mathbf w}

\newcommand{\OCal}{\mathcal{O}}

\newcommand{\SCal}{\mathcal{S}}

\newcommand{\br}{\mathbb{R}}

\newcommand{\ba}{\begin{array}}
\newcommand{\ea}{\end{array}}

\newcommand{\NCal}{\mathcal{N}}

\newcommand*{\indep}{%
	\rotatebox[origin=c]{90}{$\models$}
}

\newcommand{\N}{\mathcal{N}}
\newcommand{\OR}{\textsc{OR}}
\newcommand{\CV}{\textsc{CV}}

\newcommand{\wg}[1]{\textcolor{teal}{[WG: #1]}}

\newcommand{\bigO}{\mathcal{O}}

\newcommand{\littleO}{o}

%% file: sec/abstract.tex
While many areas of machine learning have benefited from the increasing
availability of large and varied datasets, the benefit to causal
inference has been limited given the strong assumptions needed to ensure identifiability of causal effects; these are often not satisfied in real-world datasets. For example, many large observational datasets (e.g., case-control studies in epidemiology, click-through data in recommender systems) suffer from selection bias on the outcome, which makes the average treatment
effect (ATE) unidentifiable. We propose a general
algorithm to estimate causal effects from \emph{multiple} data sources, where
the ATE may be identifiable only in some datasets but not others. The
key idea is to construct control variates using the datasets in which the
ATE is not identifiable. We show theoretically that this reduces the variance of the ATE estimate. We apply this framework to inference from observational data under outcome selection bias, assuming access to an auxiliary small dataset from which we can obtain a consistent estimate of the ATE. We construct a control variate by taking the difference of the odds ratio estimates from the two datasets. Across simulations and two case studies with real data, we show that this control variate can significantly reduce the variance of the ATE estimate.

%% file: sec/intro.tex
\section{Introduction}

The ongoing rapid growth in the scale and scope of data sources has challenged many research communities, including optimization, machine learning, and causal inference~\citep{hand2007principles,agarwal2012distributed,bottou2018optimization,shiffrin2016drawing}. In particular, in causal inference, there has been a surge of interest in developing tools to draw causal conclusions from large-scale observational data~\citep{imbens2015causal,pearl2009causality,maathuis2010predicting,kleinberg2013causal,maslove2019causal,wachinger2019quantifying}.  Compared to randomized trial designs, these observational data sources can often offer longitudinal data, fine-grained measurements, and much larger sample sizes. For example, electronic health data, including electronic health records (EHR) used for clinical care, may contain extensive details that include the timing, intensity, and quality of the interventions received by individuals. Moreover, randomized clinical trials are sometimes not feasible due to logistical, economic, or ethical reasons, and even when feasible can be seriously limited in terms of sample size~\citep{stuart2013estimating}. Thus, large-scale observational datasets hold open the promise of a much greater impact for causal inference methodology. 

However, conceptual problems arise in the observational data setting which can make causal effects unidentifiable or hard to estimate. The problems include unmeasured confounding, noisy measurements, inconsistency, and selection bias~\citep{rosenbaum1983central,angrist1996identification,nalatore2007mitigating,hernan2004structural}. These problems have generally been studied in the setting of a single observational data source, and in such a setting it is natural to view them through an all-or-nothing lens---either selection bias is present or it is not, either confounding is present or it is not, etc.  In such cases, causal inference is possible only when the data source satisfies certain delicate assumptions. These assumptions are often invalid; in particular, selection bias is notorious for being difficult to assume away---for example, in case-control datasets in epidemiological studies, cases are much more likely to be reported than non-cases; in observational data in recommender systems, certain items are more likely to receive clicks and ratings~\citep{rothman2008modern, robins2000marginal,robins2001data,  hernan2004structural,xuanhui2016learningtorank,schnabel2016recommendations,wang2020causal}. Under the existence of such selection bias, causal effects are in general unidentifiable.

In this paper, we take a different route and consider estimating causal effects from \textit{multiple} data sources. Can we combine large, possibly biased datasets with smaller, unbiased datasets to develop efficient estimators of causal effects? Our work is motivated by the observation that, in practice, we might be able to obtain a small dataset where the causal effects are identifiable; for example, from small-scale randomized trials or observational data with limited known confounding. Causal inference may not be efficient in such small datasets alone, due to limited sample size.  The hope is that large observational datasets, while not permitting causal inference by themselves, may be useful in improving the efficiency of the causal effects estimators from the small, unbiased dataset.

\paragraph{Contributions} We present an affirmative answer to this question in this paper, proposing a general framework to estimate causal effects from multiple data sources. Specifically, we show how to perform inference when the average treatment effect (ATE) may be identifiable only in some datasets but not others. The key is to use certain features that are robust and transportable across datasets to construct meaningful control variates. We show that such control variates allow us to design new ATE estimators which enjoy variance reduction, and we theoretically quantify the reduction in variance. To demonstrate the effectiveness of this general approach, we apply it to case-control studies with outcome selection bias. In particular, we construct control variates by taking the difference of the odds ratio estimates across datasets. We also present experiments, using synthetic data and two real-data case studies, to show that this control variate can significantly reduce the variance of ATE estimates using a variety of estimators.

%% file: sec/related_work.tex
\paragraph{Related work.} The problem of combining multiple datasets to estimate causal effects has attracted much recent research interest given the strong practical incentives, especially combining datasets from observational and experimental sources~\citep{colnet2020causal,rosenfeld2017predicting,rosenman2018propensity,rosenman2020combining, kallus2018removing}, where the observational data may suffer from hidden confounders and complex patterns of missing data. \citet{yang2020combining} propose estimators by combining a main dataset with unmeasured confounders and a smaller validation dataset with supplementary information on these confounders; \citet{cannings2019correlation} propose new estimators by combining datasets with complete cases and further observations with missing values, in order to improve on the performance of the complete-case estimator; 
Their technique is based on constructing multiple error-prone estimators that are transportable across the main and validation datasets. In fact, their ``error-prone estimators'' are used to design one particular choice of control variates in our setting. However, their estimators rely on the identifiability of the ATE in observational data, and therefore do not handle selection bias. 


Selection bias is induced by preferential selection of data points, and it is often governed by unknown factors that can interact with treatments, outcomes and their consequences. Operationally, selection bias cannot be easily eliminated by random sampling. There have been extensive studies on methods that deal with mitigating certain selection biases in observational studies. \citet{bareinboim2012controlling} discuss graphical and algebraic methods, and derive a general condition together with a procedure for recovering the odds ratio under selection bias. They also propose using instrumental variables for the removal of selection bias in the presence of confounding bias. \citet{zhang2008completeness} studies special cases in which selection bias can be detected even from the observations, as captured by a non-chordal undirected graphical component. \citet{robins2000marginal} and \citet{hernan2004structural} propose epidemiological methods that assume knowledge of the probability of selection given treatment, which can be estimated from data in certain cases.

The control variates technique is a classical tool for variance reduction, and there have been applications of it to causal effect estimation \citep{tan2006distributional}. Here we use the control variates technique for variance reduction in multi-source causal inference. The key technical development of our approach is the design of a valid control variate for multi-source causal inference. To this end, we propose to identify an estimand that is transportable between the observational data---i.e., the selection-biased dataset---to the experimental data. If this estimand has sufficient correlation with the target estimand of interest, it can be used to construct control variates. To this end, our work relates to the literature on transportability in causal inference \citep{bareinboim2014transportability,lee2020general,bareinboim2016datafusion}. These works investigate what causal quantities are identifiable, which suggests potential ways to construct control variates.

%% file: sec/prelim.tex
\section{Preliminaries}\label{sec:setup}

In this section, we present the basic setup for causal inference with multiple data sources. We also formalize our assumptions on the identification of causal effects.



\paragraph{Potential outcomes and ATE estimation.}
We use the potential outcomes framework to define causal effects \citep{neyman1923applications, rubin1974estimating}. Let $Z$ denote a binary treatment random variable, with $0$ and $1$ being the labels for control and active treatments,
respectively. For each realization of the level of treatment $z\in\{0,1\}$, we assume that there exists a potential outcome $Y(z)$ representing the outcome had the subject been given treatment
$z$ (possibly contrary to fact). Then, the observed outcome is $Y=Y(Z)=ZY(1)+(1-Z)Y(0)$. Further, we denote a vector of observed pretreatment covariates as $X \in \br^d$. We focus on estimating the ATE: $\tau = E[Y(1) - Y(0)]$.

\paragraph{Data sources.}
We consider a main data source that consists of observations $\OCal_1 =\{(Z_i, X_i, Y_i) : i \in \SCal_1\}$, with sample size $n_1 = |\SCal_1|$, and a validation data source with observations $\OCal_2 =\{(Z_j, X_j, Y_j) : j \in \SCal_2\}$ and sample size $n_2 = |\SCal_2|$. We assume that the ATE is identifiable from the validation data source but not necessarily the main data source, and generally $n_2 < n_1$. For simplicity we consider two data sources, but it is straightforward to generalize to more data sources.

A fundamental problem in causal inference is that the counterfactuals are not observable. Therefore, to allow for the identification of ATE, we make the following ignorability assumption \citep{rosenbaum1983central} with respect to the validation data $\OCal_2$.
\begin{assumption}[Ignorability]\label{assump-ignorable} $Y(z)\indep Z\mid X$
	for $z=0$ and $z=1$.
\end{assumption}

Under Assumption \ref{assump-ignorable}, many methods for estimating the ATE from a single observational dataset exist in the causal inference literature~\citep[see, e.g.,][]{rosenbaum2002,imbens2004,rubin2006}.

%% file: sec/methods.tex
\section{Control Variates for Multi-Source Causal Inference: A General Strategy} \label{sec:methods}

We present a general strategy for efficient estimation of the ATE by utilizing both the main and validation data. Such a strategy allows us to design efficient ATE estimators using all the data, without requiring the ATE to be identifiable in all individual data sources. Informally, we want to identify features that are transportable across both datasets and robust to the type of confounding or bias affecting the main data. Using such features, we exploit information across the datasets and improve the efficiency of the ATE estimator using all datasets. This strategy is reminiscent of a control-variate methodology for general variance reduction in Monte Carlo simulations \cite{owen2013}.

Let $\psi \in \mathbb{R}^m$ be an estimand for which there exist consistent estimators obtainable from datasets $\OCal_1$ and $\OCal_2$ (with consistent estimators denoted by $\widehat{\psi}_1$ and $\widehat{\psi}_2$, respectively). The key requirement of \textit{transportability} is that $\widehat{\psi}_1 - \widehat{\psi}_2$ converges asymptotically to zero. Let $\widehat{\tau}_{2}$ denote a consistent estimator of the true ATE $\tau$ that we obtain using dataset $\OCal_2$ with asymptotic variance $v_2$. In particular, we consider a class of estimators satisfying 
\begin{align}\label{eq:asymp}
\begin{split}
	n_{2}^{1/2}\left(\begin{array}{c}
		\widehat{\tau}_{2}-\tau\\
		\widehat{\psi}_{2}-\widehat{\psi}_{1}
	\end{array}\right) 
	\rightarrow\N\left\{ 0,\left(\begin{array}{cc}
		v_{2} & \Gamma^{\T}\\
		\Gamma & V
	\end{array}\right)\right\},
\end{split}
\end{align}

for some $V \in \mathbb{R}^{m \times m}$ and $\Gamma \in \mathbb{R}^{m \times 1}$. If Eq.~(\ref{eq:asymp}) holds exactly rather than asymptotically, by multivariate normal theory, we have the following the conditional
distribution:
\begin{align*}
    n_{2}^{1/2}(\widehat{\tau}_{2}-\tau)\mid n_{2}^{1/2}(	\widehat{\psi}_{2}-	\widehat{\psi}_{1}) \sim \mathcal{N}\left\{ n_{2}^{1/2}\Gamma^{\T}V^{-1}(	\widehat{\psi}_{2}-	\widehat{\psi}_{1}),v_{2}-\Gamma^{\T}V^{-1}\Gamma\right\}.
\end{align*}

We apply the method of control variates~\cite{owen2013} by using the estimators for $\tau$ and $\psi$ jointly to build a new estimator of $\tau$ which has a lower variance than $\hat \tau_2$. Specifically, we construct a new estimator for ATE using control variates as follows: $\widehat{\tau}_{\CV}(\beta) = \widehat{\tau}_2 - \beta^\T (\widehat{\psi}_{2}-\widehat{\psi}_{1})$.
Solving for the optimal $\beta$, we obtain the new estimator 
\begin{equation}\label{eq:proposed-estimator-general}
	\widehat{\tau}_{\CV}=\widehat{\tau}_{2}-\Gamma^{\T}V^{-1}(\widehat{\psi}_{2}-\widehat{\psi}_{1}),
\end{equation}

where $V = \text{Var}(\widehat{\psi}_{2} - \widehat{\psi}_{1})^{-1}$ and $\Gamma = \text{Cov}(\widehat{\psi}_{2} - \widehat{\psi}_{1}, \widehat{\tau}_2)$.

\begin{theorem}(\citet{owen2013,yang2020combining})\label{thm:general-cv-var} Denote the asymptotic variance of $\widehat{\tau}_{2}$ as $v_2$. Under Assumption \ref{assump-ignorable}, if Eq.~\eqref{eq:asymp} holds, then $\widehat{\tau}_{\CV}$ is consistent for $\tau$, and we have: $n_{2}^{1/2}(\widehat{\tau}_{\CV}-\tau)\rightarrow\N(0,v_{2}-\Gamma^{\T}V^{-1}\Gamma),$
in distribution as $n_{2}\rightarrow\infty$ with ratio $ n_2 / n_1 \to \rho \in [0,1]$ converging to a constant. 
Given a nonzero $\Gamma$, the asymptotic variance, $v_{2}-\Gamma^{\T}V^{-1}\Gamma,$ is smaller than $v_2$.
\end{theorem}

From a practical standpoint, Theorem \ref{thm:general-cv-var} shows that the most effective control variate estimators $\widehat{\psi}_1 - \widehat{\psi}_2$ will have low variance and high correlation with the ATE estimator $\widehat{\tau}_2$. Empirically, to estimate the optimal value  of $\beta$, we can use estimators $\widehat{V}$ and $\widehat{\Gamma}$ for the variance and covariance in Eq.~\eqref{eq:proposed-estimator-general}. These estimators $\widehat{V}$ and $\widehat{\Gamma}$ can be obtained by bootstrap sampling from the given data, for which we present details in Appendix~\ref{app:bootstrap}.

%% file: sec/estimate_OR.tex
\section{Control Variates for Outcome Selection Bias}\label{sec:cv_selection_bias}

We now present the constructions of new control variates to improve the efficiency of ATE estimates when the data suffer from selection bias on the outcome. Such selection bias on the outcome occurs in practice in case-control studies in epidemiology \cite{rothman2008modern, robins2001data, robins2000marginal}, and in recommender systems as a problem with implicit feedback \cite{xuanhui2016learningtorank,schnabel2016recommendations,wang2020causal}. However, current methodology for utilizing such selection-biased data for causal inference is limited. While we focus on selection bias for the rest of the paper, we discuss applications of the control variates strategy to other data settings in Section~\ref{sec:conclu}. 

We instantiate the main data $\mathcal{O}_1$ as observational data that suffers from selection bias on the outcome, from which the ATE is unidentifiable. Under selection bias on the outcome, we show that we are able to obtain consistent odds ratio estimates in various models, which can then be used to construct control variates across all datasets. 

\textbf{Outcome selection bias.} Outcome selection bias is induced by preferential selection of units based on the outcome. To illuminate the nature of this bias, consider the model of Figure \ref{fig:selection_bias_y}, in which $S \in \{0,1\}$ represents the selection mechanism: $S = 1$ means presence in the sample, and $S = 0$ means absence. Recall that
$X$ represents the pre-treatment covariates, $Z$ represents a binary treatment, and $Y$ represents a binary outcome. 

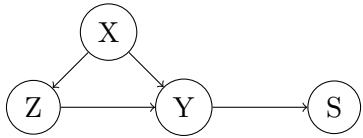
\begin{figure}[!h]
    \centering
    \begin{tikzpicture}
    \node[shape=circle,draw=black] (Z) at (0,0) {Z};
    \node[shape=circle,draw=black] (X) at (1,1) {X};
    \node[shape=circle,draw=black] (Y) at (2,0) {Y};
    \node[shape=circle,draw=black] (S) at (4,0) {S};
    \path [->] (X) edge node[left] {} (Y);
    \path [->] (Z) edge node[left] {} (Y);
    \path [->] (X) edge node[left] {} (Z);
    \path [->] (Y) edge node[left] {} (S);
    \end{tikzpicture}
    \caption{Causal graph for dataset $\OCal_1$ where
    there is a selection bias that depends on the outcome.}
    \label{fig:selection_bias_y}
\end{figure}

Under the existence of such selection bias on $Y$, the ATE is in general not identifiable, even if we assume that there are no unobserved confounding between $Z$ and $Y$. Thus, we consider the main data $\OCal_1$ to be a dataset suffering from outcome selection bias, and the validation data $\OCal_2$ to be unbiased (e.g., obtained from small-scale randomized trials).

\textbf{Variance reduction with the odds ratios.} Note that Eq.~\eqref{eq:proposed-estimator-general} applies to any estimand $\psi$, as long as it is correlated with ATE and we are able to obtain consistent estimators of it from both datasets. We show that when $\OCal_1$ suffers from outcome selection bias, we can use the odds ratios as the choice of $\psi$. This is based on two observations: First, the estimates for the conditional odds ratios and the ATE are correlated; second, the conditional odds ratios are robust to outcome selection bias under a variety of data generating processes. 


\begin{definition}\label{def:odds}
The conditional odds ratio between a binary treatment $Z$ and a binary outcome $Y$ conditioned on covariates $X$ is (for brevity, $x$ denotes $X = x$):
\begin{align*}
 \OR(x) = \frac{P(Y(1) = 1 | x) P(Y(0) = 0 |  x)}{P(Y(1) = 0 | x) P(Y(0) = 1 | x)}.
\end{align*}
\end{definition}
%


Under Assumption \ref{assump-ignorable}, we have $P(Y(1) = 1 |x) = P(Y(1) = 1 | Z=1,x) = P(Y = 1 | Z=1,x)$. Therefore, we can rewrite $\OR(x)$ in Definition~\eqref{def:odds} as: 
\begin{align*}
    \begin{split}
         &\OR(x) = \frac{P(Y = 1 | Z=1, x) P(Y = 0 | Z=0, x)}{P(Y = 0 | Z=1, x) P(Y = 1 | Z=0, x)}.
    \end{split}
\end{align*}

Proposition \ref{lem:or_selection_bias} further allows us to directly estimate the conditional odds ratios empirically with finite samples, even under selection bias.

\begin{proposition}\label{lem:or_selection_bias}(Proof in Appendix \ref{app:proofs})
If the selection $S$ depends solely on $Y$ (as in Figure \ref{fig:selection_bias_y}), then the conditional odds ratio is transportable and given by:
\begin{align*}
    \begin{split}
         \OR(x) = \frac{P(Y = 1 |S=1, Z=1,  x) P(Y = 0 | S=1,Z=0,  x)}{P(Y = 0 |S=1, Z=1,   x) P(Y = 1 |S=1, Z=0,   x)}.
    \end{split}
\end{align*}
\end{proposition}


Proposition \ref{lem:or_selection_bias} guarantees that the identifiability does not rely on any modeling assumption; see also \citet{didelez2010graphical} and \citet{jiang2017directions}. Therefore, to construct control variates, it is sufficient to derive consistent estimators of the odds ratio from the datasets $\OCal_1$ and $\OCal_2$. Let $\widehat{\OR}_{1}(x)$ and $\widehat{\OR}_{2}(x)$ denote consistent estimators for $\OR(x)$ obtained from the datasets $\OCal_1$ and $\OCal_2$, respectively. 
For a set of covariate values $\{x_1, \cdots, x_k\}$, one possible control variate construction is to take 
$\psi = \big(\OR(x_1) , \cdots, \OR(x_k)\big)^\top.$ Then $
\widehat{\psi}_1 = \big(\widehat{\OR}_{1}(x_1) , \cdots, \widehat{\OR}_{1}(x_k)\big)^\top$, $\widehat{\psi}_2 = \big(\widehat{\OR}_{2}(x_1) , \cdots, \widehat{\OR}_{2}(x_k)\big)^\top $.
Substituting these into Eq.~\eqref{eq:proposed-estimator-general} gives the new ATE estimator with control variates. 

When $X$ is continuous or there are too many discrete values of $X$, we can reduce the large number of conditional odds ratios $\OR(x)$ down to a manageable control variate by integrating $\OR(x)$ over a common distribution $F(x)$: let $\psi = \int \OR(x) F(\text{d}x)$. Then $\widehat{\psi}_{1} = \int \widehat{\OR}_{1}(x) F(\text{d} x)$ and $\widehat{\psi}_{2} = \int \widehat{\OR}_{2}(x) F(\text{d} x)$. 

\subsection{Estimating the odds ratio under selection bias}

The main difficulty in applying Eq.~\eqref{eq:proposed-estimator-general} is to find consistent estimators for constructing the control variates. We demonstrate that this can be achieved by estimating the conditional odds ratios parametrically using a logistic model with varying coefficients, or non-parametrically using kernel smoothing. For the consistency analysis of estimators for ATE and the odds ratios, we assume that the data points in $\OCal_1$ and $\OCal_2$ \textit{before} selection bias are IID samples from the same underlying population for $X, Z, Y$.

\textbf{Logistic outcome model.}\label{sec:logistic_outcome} One approach to estimating the odds ratio $\OR(x)$ uses a logistic model with varying coefficients ~\cite{cleveland1991local} to parameterize the outcome distribution:
\begin{equation}\label{eq:varying-coeff}
    P(Y=1 | Z = z, x) = \frac{e^{\beta_0^x + \beta_1^x z}}{1 + e^{\beta_0^x + \beta_1^x z}}.
\end{equation}
Here, $\beta_0^x, \beta_1^x$ are coefficients that depend on the covariates $x$. If $X$ is discrete and finite, then there would be a discrete and finite number of parameters $\beta_0^x, \beta_1^x$. Otherwise, $\beta_0^x$ and $\beta_1^x$ can be viewed as functions of $x$.


If the data is truly generated by the outcome model defined in Eq.~\eqref{eq:varying-coeff}, then Theorem \ref{thm:OR-robust-selection-bias} below shows that selection bias on the outcome will not change the coefficient $\beta_1^x$ across $\OCal_1$ and $\OCal_2$. Furthermore, $\beta_1^x$ is the only parameter needed to compute the conditional odds ratio $\OR(x)$. Thus, any consistent estimates of $\beta_1^x$ for both $\OCal_1$ and $\OCal_2$ would provide consistent estimates of the conditional odds ratio that are robust to selection bias.
\begin{theorem}\label{thm:OR-robust-selection-bias}(Proof in Appendix \ref{app:proofs})
If the selection $S$ depends solely on $Y$ (as in Figure \ref{fig:selection_bias_y}) and $P(Y=1 |  Z = z, X = x)$ follows the logistic model in \eqref{eq:varying-coeff}, then $P(Y=1 | Z = z,X = x,  S = 1)$ also follows a logistic model, with the same coefficient $\beta_1^x$ on $Z$ as the logistic model for $P(Y=1 |  Z = z, X = x)$ for each covariate value $x$. Moreover, the conditional odds ratio is $\OR(x) = e^{\beta_1^x}$.
\end{theorem}


Theorem \ref{thm:OR-robust-selection-bias} extends \citet{prentice1979logistic}; see also \citet{agresti2015foundations}.  
By Theorem \ref{thm:OR-robust-selection-bias}, we need only compute consistent estimators $\widehat{\beta}_{1, \OCal_1}^x$ and $\widehat{\beta}_{1, \OCal_2}^x$ of $\beta_1^x$ from $\OCal_1$ and $\OCal_2$ to produce consistent estimators of the true underlying conditional odds ratio $\OR(x)$, $\widehat{\OR}_1(x) = e^{\widehat{\beta}_{1, \OCal_1}^x}, \widehat{\OR}_2(x) = e^{\widehat{\beta}_{1, \OCal_2}^x}$.

When $X$ is discrete, we can obtain such consistent estimators $\widehat{\beta}_{1, \OCal_1}^x$ and $\widehat{\beta}_{1, \OCal_2}^x$ by stratifying the data on $X$ and performing logistic regression within each stratum. Let $\widehat{\beta}_{1, \OCal_2}^x$ be the maximum likelihood estimator for $\beta_{1, \OCal_2}^x$ in the stratum (or subset of data) with $X = x$ from $\OCal_2$ (and $\widehat{\beta}_{1, \OCal_1}^x$ be the same for $\OCal_1$). These maximum likelihood estimators are consistent estimators for the true $\beta_{1}^x$. 

For continuous $X$, producing a theoretically consistent estimator for $\beta_{0}^x, \beta_{1}^x$ is more challenging. One technique is to assume parametric models for the functions $\beta_{0}^x = f_0(x, \theta_0)$, $\beta_{1}^x = f_1(x, \theta_1)$.
For example, if these functions are linear (i.e., $f_0(x, \theta_0) = \theta_0^\top x,$ $f_1(x, \theta_1) = \theta_1^\top x$), then the problem of estimating $\beta_{1}^x$ reduces to maximum likelihood estimation of $\theta_1$ over a logistic model:
$P(Y=1 | Z = z, x) =  e^{\theta_0^\top x + \theta_1^\top x z} / (1 + e^{\theta_0^\top x + \theta_1^\top x z}).$
We may also allow $f_0(x, \theta_0), f_1(x, \theta_1)$ to take more general functional forms, such as neural networks. Depending on the complexity of the functions, it becomes more challenging to  guarantee asymptotic consistency theoretically and obtain a rate of convergence to the true $\beta_{1}^x$. However, such methods may still work well in practice. We explore their empirical performance in Section \ref{sec:experiments_real}.

\textbf{Kernel smoothing.}\label{sec:kernel} When $X$ is continuous, we can also estimate the odds ratio using kernel smoothing without making any parametric assumptions on the exact outcome model or functional form of $\beta_1^x$. First, notice that
$
\OR(x) 
= \frac{\E[YZ|x]\cdot \E[(1-Y)(1-Z)|x]}{\E[Y(1-Z)|x]\cdot \E[(1-Y)Z|x]}.
$
Further,  by Proposition \ref{lem:or_selection_bias},
$$
\OR(x)  
	= \frac{\E[YZ|S=1, x]\cdot \E[(1-Y)(1-Z)|S=1,x]}{\E[Y(1-Z)|S=1,x]\cdot \E[(1-Y)Z|S=1,x]}.
$$
Therefore, estimating $\OR(x)$ is equivalent to estimating $\E[W|x]$ and $\E[W|S=1, x]$ from $\OCal_1$ and $\OCal_2$, respectively, where $W \in \{YZ, (1-Y)(1-Z), Y(1-Z), Z(1-Y)\}$. Choose a kernel function $K(\cdot)$ and the bandwidth $\lambda$.
Given a dataset with $n$ data points $(X_i, Y_i, Z_i)_{i=1}^n$, for a random variable $W \in \{YZ, (1-Y)(1-Z), Y(1-Z), Z(1-Y)\}$, $\hat \E[W|x] =  \sum_{i=1}^N K(\frac{x-X_i}{\lambda}) W_i / \sum_{i=1}^N K(\frac{x-X_i}{\lambda}) .$
Therefore, the kernel estimator is:
\begin{align*}
    	\widehat{\OR}(x) = \frac{\sum_{i=1}^N K(\frac{x-X_i}{\lambda})Y_i Z_i \sum_{i=1}^N K(\frac{x-X_i}{\lambda})(1-Y_i)(1-Z_i)}{\sum_{i=1}^N K(\frac{x-X_i}{\lambda})Y_i(1-Z_i)K(\frac{x-X_i}{\lambda})(1-Y_i)Z_i}.
\end{align*}

This estimator is consistent under selection bias on the outcome as shown by Proposition \ref{lem:or_selection_bias}.
However, unlike the parametric estimators using the MLE, we note that the asymptotic convergence of the kernel estimator depends on the bandwidth $\lambda$ and the dimensionality $d$. We provide further analysis of this convergence for the odds ratio in Appendix \ref{app:kernel}.

%% file: sec/experiments.tex
\section{Simulation Experiments}\label{sec:experiments_sim}

We first demonstrate the finite-sample performance of estimators with and without the proposed control variates in a simulation study. We simulate an observational dataset with confounding from $X$ using a logistic model adapted from \citet{zhang2009estimatingoddsratio}. 



\textbf{Data generation.} We generate the dataset $\OCal_2$ by sampling $n_2$ samples from the following data-generating process. Let $X \in \mathbb{R}^2$ have two components $X_1, X_2$, which are IID Bernoulli($p=0.5$). Given $X$, the treatment assignment $Z$ is distributed as 
$P(Z = 1 | X = x) =  e^{a_0 + a_1^{\top}x} / ( 1 + e^{a_0 + a_1^{\top}x}).$ As done by \citet{zhang2009estimatingoddsratio}, the outcome $Y$ is generated from a logistic model with an interaction term between $X$ and $Z$ parameterized by $\{\beta_i\}_{i=0}^3$:
\begin{equation}\label{eq:logistic_interaction}
    P(Y = 1 | Z = z, x) = \frac{e^{\beta_0 + \beta_1 z + \beta_2^{\top} x + \beta_3^{\top} xz}}{1 + e^{\beta_0 + \beta_1 z + \beta_2^{\top} x + \beta_3^{\top} xz}}.
\end{equation}
We specifically set $\beta_3 \neq 0$ so that the conditional odds ratio varies as a function of $x$. Full details with exact parameters settings are given in Appendix \ref{app:experiments_sim_data}.
To generate  $\OCal_1$, we first draw $(Z_i, X_i, Y_i)_{i=1}^N$ samples from the same data-generating process as  $\OCal_2$, and include each sample $(Z_i, X_i, Y_i)$ in $\OCal_1$ with probabilities $P(S_i = 1 | Y_i = 1) = 0.9$ and $P(S_i = 1 | Y_i = 0) = 0.1$. This simulates selection bias in favor of positive outcomes as happens in case-control studies in practice. 

\begin{figure*}[!ht]
\centering
\begin{tabular}{ccc} 
\includegraphics[width=0.31\textwidth]{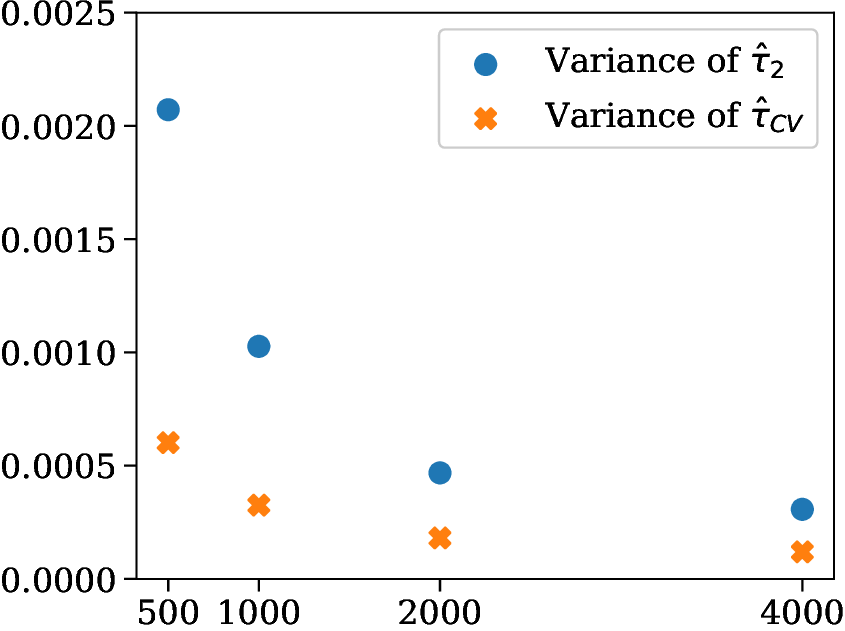} & 
\includegraphics[width=0.31\textwidth]{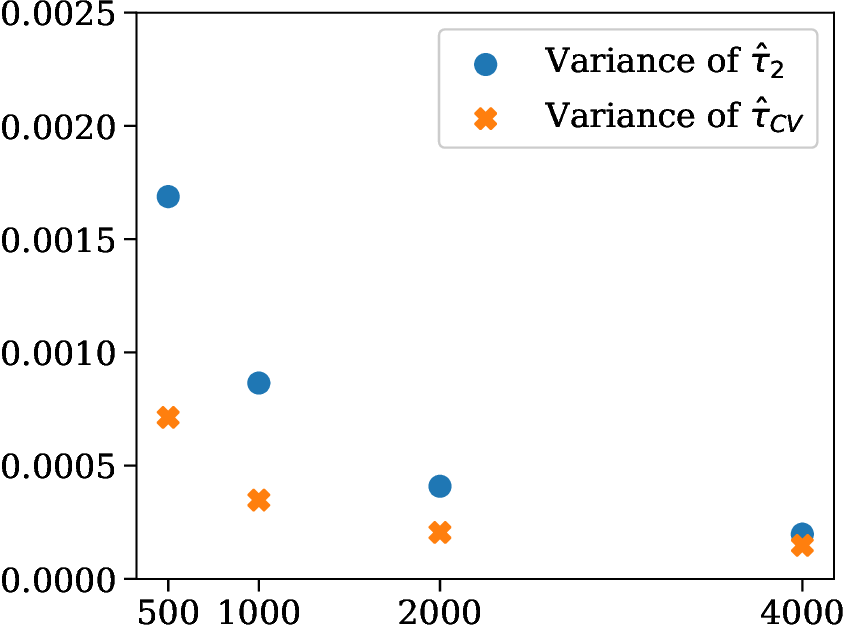} & 
\includegraphics[width=0.31\textwidth]{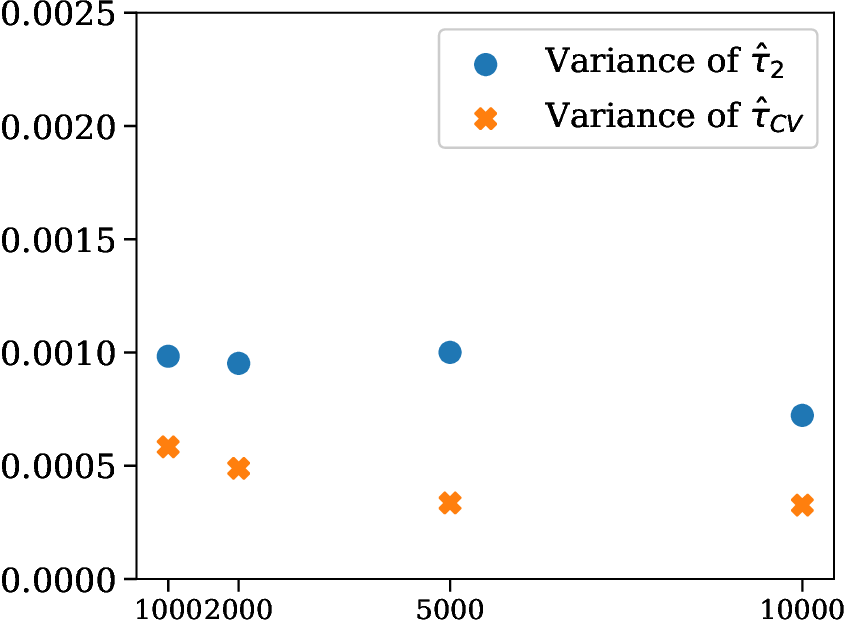}\\
\scriptsize{$n_2$} & \scriptsize{$n_2$} & \scriptsize{$n_2$} \\
\end{tabular}
\caption{\footnotesize{\textbf{Scenario 1 (left):} Comparisons of variance as the size of the observational dataset $n_2$ increases under the general logistic model. The ratio $n_2/n_1$ is kept constant at $1/10$. \textbf{Scenario 2 (middle):} Comparisons of variance as the size of the observational dataset $n_2$ increases under the general logistic model. The size of the selection bias dataset is fixed at $n_1=10000$. \textbf{Scenario 3 (right):} Comparisons of variance as the size of the selection bias dataset $n_1$ increases under the general logistic model. The size of the observational dataset is fixed at $n_2=1000$. \textit{Lower is better.}}}\label{fig:var_reg} 
\end{figure*}

\textbf{Estimating the ATE.} 
To obtain an estimate $\widehat{{\tau}}_{2}$ of the ATE from $\OCal_2$, we use a parametric imputation estimator. Denote the coefficient and intercept resulting from logistic regression of $Y$ on $Z$ for stratum $X=x$ as $\widehat{\beta}_1^{x}$ and $\widehat{\beta}_0^{x}$. The regression imputation estimator of the ATE is given by:
\begin{equation}\label{eq:regression_ate}
    \widehat{\tau}_{2} = n_2^{-1}\sum_{i=1}^{n_2} \bigg \{\frac{e^{\hat \beta_0^{X_i} + \hat \beta_1^{X_i}}}{1+e^{\hat \beta_0^{X_i} + \hat \beta_1^{X_i}}}  - \frac{e^{\hat \beta_0^{X_i}} }{1+e^{\hat \beta_0^{X_i}}}\bigg \}.
\end{equation}
This logistic regression model is well specified as it coincides with the true data-generating model.


\textbf{Estimating the control variate.}
To estimate the conditional odds ratio, we perform logistic regression of $Y$ on $Z$ on each stratum with $X = x$ to obtain estimates of $\beta_1^{x}$ from both $\OCal_1$ and $\OCal_2$, which produces estimates $\widehat{\OR}_1(x), \widehat{\OR}_2(x)$ as described in Section \ref{sec:logistic_outcome}. 
To compute the proposed control variates estimator $\widehat{\tau}_{\CV}$ (Eq.~\eqref{eq:proposed-estimator-general}), we ran $B=100$ bootstrap replicates to estimate the covariances $\Gamma$ and $V$, which we use estimate the optimal coefficient $\widehat{{\Gamma}}^{\top} \widehat{{V}}^{-1}$ for the control variate.

\subsection{Results}

Figure \ref{fig:var_reg} compares the variance for the ATE estimator $\widehat{{\tau}}_{2}$ with the variance of the ATE estimator with control variates $\widehat{{\tau}}_{\CV}$ for three different finite-sample scenarios varying $n_1$ and $n_2$ (full scenario descriptions in Appendix \ref{app:experiments_sim_scenarios}). The variances of these estimators are measured over $B=100$ bootstrap replicates. Throughout all three scenarios, the estimator with control variates $\widehat{{\tau}}_{\CV}$ had significantly reduced variance compared to $\widehat{{\tau}}_{2}$ alone. 
However, the impact of increasing $n_1$ and $n_2$ varies, with $n_2$ mattering much more for improving the variance. Figure \ref{fig:var_reg} shows that the variance of $\widehat{{\tau}}_{2}$ and $\widehat{{\tau}}_{\CV}$ both decrease significantly as $n_2$ increases, even if the ratio $n_2/n_1$ is not necessarily fixed. Consistent with these findings, in Figure \ref{fig:var_reg}, we also observe that when there is a limited fixed amount of observational data $n_2$, increasing the amount of selection-biased data does not seem to significantly improve the variance of the estimator with control variates, $\widehat{{\tau}}_{\CV}$. We further report the bias of each estimator over the bootstrap replicates in Appendix \ref{app:experiments_sim} (Figures \ref{fig:bias_reg_ratio-fixed_varying-coef},  \ref{fig:bias_reg_n1-fixed_varying-coef}, and \ref{fig:bias_reg_n2-fixed_varying-coef}). In general, the bias decreases as $n_2$ increases, and is not significantly different with or without control variates.

\section{Real-Data Case Studies}\label{sec:experiments_real}

In addition to the full simulation, we evaluated the performance of the proposed control variates on two case studies with public datasets. All code will be made publicly available. 

\textbf{Case study 1: flu shot encouragement with selection bias from case-control studies:} We consider a flu shot encouragement experiment dataset that has been repeatedly studied in the causal inference literature \citep{mcdonald1992effects,hirano2000assessing,ding2017principal}. As done in prior work, we use data from 1980 for $2,861$ patients collected from an encouragement experiment in which participating physicians were assigned treatments $Z$ uniformly at random, where $Z = 1$ indicates that the patient's physician was sent a letter encouraging them to vaccinate their patients (and $Z=0$ otherwise). The binary outcome $Y$ is whether the patient was hospitalized for flu-related reasons the following winter. The dataset contains eight additional covariates $X$, one of which is a continuous \textit{age} variable, and seven of which are binary indicators of prior medical conditions of the patient (e.g., history of heart disease). We do not consider the intermediate variable of whether or not the patient received a flu shot.

We set $\mathcal{O}_2$ to be this original dataset.
For the second dataset $\mathcal{O}_1$, we consider a realistic scenario where an additional observational dataset exists consisting mostly of patients who have already been hospitalized for flu-related reasons. Observational studies consisting mostly of positive outcomes are common in epidemiology as case-control studies. 
Such a dataset may be easier to collect than the original encouragement experiment, since such data may already be available from hospitals without setting up an explicit controlled experiment.
We simulate $\mathcal{O}_1$ by training a logistic model with interaction terms parameterized by Eq.~\eqref{eq:logistic_interaction} on the original dataset $\mathcal{O}_2$,
and generating samples according the fitted distribution (details in Appendix \ref{app:exp_details}). We assume that the ``true'' ATE is given by this model.

\textbf{Case study 2: spam email detection with selection bias from implicit feedback:} For a second case study, we use a dataset constructed for the Atlantic Causal Inference Conference (ACIC) 2019 Data Challenge based on the Spambase dataset for spam email detection from UCI \citep{acic2019,DuaUCI}. The dataset consists of emails with outcome of interest $Y$ being whether or a user marked the email as spam. The treatment $Z$ is whether or not the email contains more than a given threshold of capital letters (the threshold is computed by a mean over the full dataset). There are 22 continuous covariates $X$ which are word frequencies given as percentages. The ACIC competition does not use the original data from UCI directly, but instead generates modified versions using pre-specified data generating processes with known true ATE. 
We generate $\mathcal{O}_2$ using ACIC's data generating process, for which we provide more details in Appendix \ref{app:exp_details}.

For $\mathcal{O}_1$, we generate data from the same data generating process and apply selection bias $P(S = 1|Y = 1) = 0.9$ and $P(S=1| Y=0) = 0.1$ to produce $n_1 = 30,000$ examples. This simulates a practical scenario where a user marking an email as spam constitutes explicit feedback, but a user \textit{not} taking action to mark an email as spam constitutes \textit{implicit} feedback. This implicit feedback is unreliable, since if a user does not mark an email as spam, there is no guarantee that the user actually even read the email in full. This implicit feedback problem and resulting selection bias has been repeatedly identified as a fundamental challenge in learning-to-rank systems and recommender systems \citep{xuanhui2016learningtorank,schnabel2016recommendations,wang2020causal}. Disregarding the unreliable implicit feedback, the resulting dataset containing only explicit feedback is subject to selection bias on the outcome $Y=1$ of being marked as spam, rendering the ATE non-identifiable. In our experiment, $\mathcal{O}_2$ is assumed to be a small curated dataset without the implicit feedback problem, which may be constructed by, e.g., asking users to explicitly mark emails as ``not spam.'' We ran two experiments to illustrate the effects of the size of this curated dataset, one with $n_2 = 3,000$ and one with a larger $n_2 = 10,000$.

\subsection{Estimators and implementation}

To handle the continuous $X$ values, we estimate the conditional odds ratio for a finite set of values $\mathcal{X}$ taken from $\mathcal{O}_2$, and set the control variate $\psi$ to be an average over all of these conditional odds ratios. Furthermore, for these datasets with continuous covariates $X$, we found that using the log conditional odds ratio was more effective as a control variate as it had lower variance for extreme values of $X$. We report results with the control variate estimand $\psi =  |\mathcal{X}|^{-1} \sum_{x \in \mathcal{X}} \log \OR(x)$. To estimate the ATE $\widehat{\tau}_2$ and the odds ratios for the control variates, we apply three different methods:

\textbf{Logistic model with interaction:} We first apply a logistic model with an interaction term between $X$ and $Z$ (Eq.~\eqref{eq:logistic_interaction}) to estimate both the ATE and conditional odds ratios (details in Appendix \ref{app:exp_details}). Since this is the same model used to generate the flu shot encouragement dataset $\OCal_1$, there is no model misspecification when using this estimator for case study 1. We set $\mathcal{X}$ to be all $X_i$ in $\OCal_2$.

\textbf{Neural network:} To allow for more flexibility, we use a neural network to estimate a logistic outcome model with varying coefficients (Eq.\eqref{eq:varying-coeff}), where $\beta_0^x = f_0(x; \theta)$, $\beta_1^x = f_1(x; \theta)$ are outputs of the neural network with parameters $\theta$. The optimization objective is the logistic loss on the final outcome prediction, and we choose the neural network architecture using five-fold cross validation (more details in Appendix \ref{app:exp_details}). We estimate both the ATE and the odds ratio using this neural network varying coefficient model. We only report results with the neural network on the spam email dataset since the flu dataset contains a small number of mostly binary covariates, and the added flexibility from the neural network does not provide much additional benefit. We set $\mathcal{X}$ to be all $X_i$ in $\OCal_2$.

\textbf{Kernel smoothing:} As a third technique, we estimate the ATE using the logistic model in Eq.~\eqref{eq:logistic_interaction}, but apply kernel smoothing to estimate the odds ratios for the control variate (as in Section \ref{sec:kernel}). This non-parametric estimate gets around any problems of model misspecification when estimating the odds ratio.
We set $\mathcal{X}$ to be a random sample of $50$ values of $X_i$ from $\OCal_2$. 

As done in the simulation study, we ran $B=300$ bootstrap replicates to estimate the variance of the ATE estimate and the covariance between the ATE estimate and the odds ratio control variates, which we use to compute the optimal coefficient $\widehat{{\Gamma}}^\top \widehat{{V}}^{-1}$ for the control variate. 

\subsection{Results}

Table \ref{tab:flu_acic_var} reports the variances of the ATE estimators with and without the proposed control variates, where the variance is computed over $B=300$ bootstrap replicates. For both case studies with all estimator combinations, the variance is reduced with the control variates. However, the amount that the variance is reduced varies significantly. For the flu dataset, the variance reduction is significantly higher than for the spam email dataset. This is likely because on the flu dataset, there is no model misspecification, as we generate $\mathcal{O}_1$ using the same logistic regression model as we use to estimate the ATE and odds ratios. 

On both case studies, kernel smoothing performs better than the logistic and neural network estimators. The kernel smoothing estimator performs significantly better than the explicit modeled estimators on the spam email dataset, which we hypothesize is due to model misspecification for the logistic regression model, and overfitting for the neural network model. Interestingly, the variance reduction of the neural network is better than the simpler logistic model when there are more samples $n_2$, which suggests possible overfitting for the neural network when $n_2$ is too small.


\begin{table}[!ht]
\caption{Variances for both the flu shot encouragement data with $n_1=10,000$ (Case Study 1) and the spam email detection data with $n_1 = 30,000$ (Case Study 2).}
\label{tab:flu_acic_var}
\centering
\resizebox{\columnwidth}{!}{
\begin{tabular}{lcc|ccc|ccc}
\toprule
& \multicolumn{2}{c}{Case Study 1} & \multicolumn{3}{c}{Case Study 2: $n_2 = 3,000$} & \multicolumn{3}{c}{Case Study 2: $n_2 = 10,000$} \\
\cmidrule{2-9}
Metric & Logistic & Kernel & Logistic & NN & Kernel & Logistic & NN & Kernel \\
\midrule
Var (\% diff) & $71.160$\%  & $77.599$\% & $3.522$\%  & $0.638$\% & $21.231$\% & $0.874$\%  & $5.940$\% & $23.134$\%\\
Var $\widehat{\tau}_2$ & \num{1.023E-04} & \num{1.072E-04} & \num{4.576e-4} & \num{1.074e-3} & \num{4.309e-4} & \num{1.443e-4} & \num{4.762e-4} & \num{1.351e-4}\\
Var $\widehat{\tau}_{\CV}$ & \num{2.952E-05} & \num{2.400E-05} & \num{4.415e-4} & \num{1.067e-3} & \num{3.394e-4} & \num{1.430e-4} & \num{4.479e-4} & \num{1.038e-4} \\
\bottomrule
\end{tabular}
}
\vskip -0.1in
\end{table}

Variance reduction sometimes came at a cost of higher bias, which was more pronounced on the spam email dataset (Table \ref{tab:acic_bias} in the Appendix). This speaks to the care that is needed in managing the bias/variance tradeoff via control variates when there is possible model misspecification.

%% file: sec/conclu.tex
\section{Conclusions and Further Connections} \label{sec:conclu}

Combining multiple data sources has the potentials of mitigating bias and improving efficiency in causal inference. We provide a general framework for combining multiple data sources for more efficient estimates of causal effect. We further illustrate this framework to improve ATE estimates given data subject to outcome selection bias, a common complication in epidemiology with limited existing methodology. 



\citet{yang2020combining} study the problem of combining multiple observational data sources with possibly unmeasured confounding, and propose a technique for this setting which can be mapped exactly onto our control variates framework. They consider a setting with two observational datasets where one large dataset contains unmeasured confounding, while another smaller dataset contains supplementary information on the confounders. They assume the two datasets have the same confounding structure (both observed and unobserved), and construct control variates which are error-prone estimators for the ATE that are transportable across the datasets. However, their approach cannot deal with outcome selection bias, which is an important problem in empirical research. We provide a practical solution to this problem based on our framework.


Our control variates framework is also applicable to other data combination settings. For example, combining randomized control trials with observational studies is another practically important problem that has received significant historical and recent attention. 
We can design other control variates by finding other quantities that are transportable between the two data sources. One example is the conditional ATE defined by some observed covariates. This quantity is transportable when both datasets share the same causal connection between the covariates and the outcome; only the connection between the covariates and the treatment differs. Therefore, we can use the difference between these conditional ATEs as a control variate. When these conditional ATEs are not identified, their corresponding error-prone estimators can lead to an equally valid control variate. Full theoretical and empirical evaluation of control variates in this setting would be valuable future work.

On a broader societal level, this methodology, like many other ATE estimation approaches, should be used with careful validation in settings where model assumptions are satisfied.



%% file: sec/app_kernel.tex

\section{Further Analysis of the Kernel Estimator}\label{app:kernel}

In this section, we analyze the asymptotic properties of the kernel odds ratio estimator. We work with the log of the odds ratio for an easier analysis of the asymptotic properties.
\begin{theorem}
Given a symmetric kernel function $K(u)$ such that $\int K(u)du = 1$, $\int uK(u)du=0$, and $\lambda$ is the bandwidth. Consider the kernel odds ratio estimator as defined in Section~\ref{sec:kernel}. Denote $f(x)$ as the density of $X, X \in \br^d$ and consider $\lambda = \littleO(\frac{1}{N})^{\frac{1}{d+4}}$. Define $\vw \triangleq (YZ, (1-Y)(1-Z), Y(1-Z), Z(1-Y))^\top$, $\Sigma(x)\triangleq\cov(\vw|x)$, $g(x)= (g^1(x), g^2(x), g^3(x), g^4(x))^\top\triangleq\E[\vw|x]$, and $A(x)\triangleq(\frac{1}{g^1(x)}, \frac{1}{g^2(x)},\frac{1}{g^3(x)},\frac{1}{g^4(x)})^\top$.  Then, 
\begin{equation*}
(n_{2}\lambda^d)^{1/2}(\log\widehat \OR(x)- \log \OR(x)) \rightarrow \mathcal{N}\left(0,   A^\top(x)\Sigma(x)A(x) f(x)^{-1}\int K^2(u)du\right),
\end{equation*}
in distribution as $N \rightarrow \infty$. 


\end{theorem}

\begin{proof}
By definition, 
\begin{align}
    \OR(x) = \frac{g^1(x)g^2(x)}{g^3(x)g^4(x)}.
\end{align}
Also by definition, the kernel estimator is
\begin{align}
    \widehat \OR(x) = \frac{\hat g^1(x)\hat g^2(x)}{\hat g^3(x)\hat g^4(x)},
\end{align}
where $\hat g^j(x) = \frac{\sum_{i=1}^N K(\frac{x-X_i}{\lambda}) W^j}{\sum_{i=1}^N K(\frac{x-X_i}{\lambda})}, i = 1\cdots 4$, $N$ as the number of the samples, and $(W^1, W^2, W^3, W^4)^\top = \vw = (YZ, (1-Y)(1-Z), Y(1-Z), Z(1-Y))^\top$. Then, $\log  \widehat \OR(x)$ is simply $\log(\hat g^1(x)) + \log(\hat g^2(x)) - \log(\hat g^3(x))- \log(\hat g^4(x))$.

We first study an asymptotic analysis for $g^j(x), j = 1\cdots 4$. Then we apply the Delta method to obtain the asymptotic consistency of $\log  \widehat \OR(x)$. Denote that 
\begin{align*}
    \hat g^j(x) = \frac{\frac{1}{N\lambda^d}\sum_{i=1}^N K(\frac{x-X_i}{\lambda}) W_i^j}{\frac{1}{N\lambda^d}\sum_{i=1}^N K(\frac{x-X_i}{\lambda})} = \frac{\hat \tau^j(x)}{\hat f(x)}.
\end{align*}
Then we have
\begin{align*}
   \E[\hat f(x)] &= \E\left[\frac{1}{N \lambda^d}\sum_{i=1}^N K(\frac{x-X_i}{\lambda}) \right]\\
    &= \E\left[\frac{1}{\lambda^d}K(\frac{x-X_i}{\lambda}) \right]\\
    &= \int \frac{1}{\lambda^d}K(\frac{x-z}{\lambda}) f(z)dz.
\end{align*}

Making the change-of-variables formula for multivariate densities $u = \frac{x-z}{\lambda}$, $du = \lambda^{-d} dz$, then
\begin{align*}
   \E[\hat f(x)]= \int K(u)f(x-\lambda u)du.
\end{align*}

When $f(x)$ is continuous, bounded above and $\int K(u)du =1$, the above converges to $f(x)$ as $\lambda$ goes to $0$ by dominated convergence theorem. To compute the bias, take the second order Taylor expansion of $f(x-\lambda u)$, we have:
\begin{align*}
    &f(x-\lambda u)=f(x) - \lambda \frac{\partial f(x)}{\partial x'}u + \frac{\lambda^2}{2}tr(\frac{\partial^2 f(x)}{\partial x \partial x'}uu') + \littleO(\lambda^2)
\end{align*}

Therefore, since the kernel is symmetric, the bias is $\bigO(\lambda^2)$.

Similarly, we calculate the variance of $\hat f(x)$:
\begin{align*}
    \var(\hat f(x)) &= \var\left(\frac{1}{N}\sum_{i=1}^N \frac{1}{\lambda^d}K(\frac{x-X_i}{\lambda})\right)\\
    &= \frac{1}{N} \var\left(\frac{1}{\lambda^d}K(\frac{x-X_i}{\lambda})\right)\\
    &=\frac{1}{N}\E\left[\left(\frac{1}{\lambda^d}K(\frac{x-X_i}{\lambda})\right)^2\right]-\frac{1}{N}\left(\E\left[\frac{1}{\lambda^d}K(\frac{x-X_i}{\lambda})\right]\right)^2\\
    &=\frac{1}{N}\int \frac{1}{\lambda^{2d}}\left(K(\frac{x-z}{\lambda})\right)^2f(z)dz - \frac{1}{N}\left(\E[\hat f(x)]\right)^2\\
    &=\frac{1}{N\lambda^d}\int \left(K(u)\right)^2 f(x-\lambda u)du - \frac{1}{N}\left(\E[\hat f(x)]\right)^2\\
    &=\frac{f(x)}{N\lambda^d} \int (K(u))^2du + \littleO(\frac{1}{N\lambda^d}),
\end{align*}
where we make the change of variable $u = \frac{x-z}{\lambda}$ again. Therefore, the variance of $\hat f(x)$ is $\bigO(\frac{1}{N\lambda^d})$. Recall its bias is $\bigO(\lambda^2)$ given that the kernel is symmetric, and the optimal bandwidth equates the rate of convergence of the squared bias and variance, i.e. $\bigO((\lambda^\ast)^4) = \bigO(\frac{1}{N(\lambda^\ast)^d})$. Therefore, the optimal bandwidth is $\lambda^\ast = \bigO(\frac{1}{N})^{\frac{1}{d+4}}$.

Further, we can consider $\hat f(x)$ as an average of a triangular array: 
\begin{align*}
    \hat f(x) = \frac{1}{N}\sum_{i=1}^N Z_{in}, 
\end{align*}
with $Z_{in} = \frac{1}{\lambda^d}K(\frac{x-X_i}{\lambda})$. By the Lyapunov CLT theorem, when $\lambda \rightarrow 0$, $\N\lambda^d \rightarrow \infty$, we have
\begin{align*}
    \frac{\hat f(x) - \E[\hat f(x)]}{\sqrt{\var(\hat f(x))}} \rightarrow N(0,1),
\end{align*}
as $N \rightarrow \infty$.

Notice that 
\begin{align*}
    \frac{\hat f(x) - f(x)}{\sqrt{\var(\hat f(x))}} = \frac{\hat f(x) - \E[\hat f(x)]}{\sqrt{\var(\hat f(x))}}  + \frac{\E[\hat f(x)] - f(x)}{\sqrt{\var(\hat f(x))}}.
\end{align*}

Pick $\lambda = \littleO(\frac{1}{N})^{\frac{1}{d+4}}$ ensures that and the bias term vanishes from the asymptotic distribution and is negligible relative to
the variance. Therefore, we have
\begin{align*}
    \frac{\hat f(x) - f(x)}{\sqrt{\var(\hat f(x))}} \rightarrow N(0,1),
\end{align*}
in distribution as $N \rightarrow \infty$.

Next we analyze $\hat \tau^j(x), j = 1\cdots 4$. By definition, 
\begin{align*}
    \E[\hat \tau^j (x)] &= \E[\frac{1}{\lambda^d}K(\frac{x-X_i}{\lambda})W_i^j]\\
    &= \E[\frac{1}{\lambda^d}K(\frac{x-X_i}{\lambda})g^j(x_i)]\\
    &= \int \frac{1}{\lambda^d}K(\frac{x-X_i}{\lambda})g^j(z)f(z) dz
\end{align*}
Similar to the derivation of $\E[\hat f(x)]$, we have 
\begin{align*}
    \E[\hat \tau^j(x)] &= g^j(x)f(x) + \bigO(\lambda^d)\\
    &\rightarrow g^j(x)f(x),
\end{align*}
as $N \rightarrow \infty$, and 
\begin{align*}
     \var(\hat  \tau^j(x)) &= \var\left(\frac{1}{N}\sum_{i=1}^N \frac{1}{\lambda^d}K(\frac{x-X_i}{\lambda}) W^j_i\right)\\
     &= \frac{\E[(W^j)^2|x]}{N\lambda^d} \int K^2(u)du + \littleO(\frac{1}{N\lambda^d}).
\end{align*}

As $N\lambda^d \rightarrow \infty$, $\begin{pmatrix}\hat \tau(x) \\ \hat f(x)\end{pmatrix}$ are jointly normal. Applying the Delta method, we have
\begin{align*}
    (N\lambda^d)^{\frac{1}{2}}(\hat g^j(x) - g^j(x)) \rightarrow \mathcal{N}\left(0, \frac{\var(W^j_i|x_i=x)}{f(x)}\int K^2(u)du\right).
\end{align*}

Therefore, for $j = 1\cdots 4$, we have $\var(\hat g^j(x)) = \frac{\var(W^j_i|x_i=x)}{f(x)}\int K^2(u)du, $. Similarly, we have that $\cov(\hat g^j(x), \hat g^k(x)) = \frac{\cov(W^j_i, W^k_i|x_i=x)}{f(x)}\int K^2(u)du$ for all $j,k = 1\cdots 4, j\neq k$ as $N$ goes to $\infty$. Thus,
\begin{equation}\label{eq:asymp-g}
	(n_{2}\lambda^d)^{1/2}\left(\hat g(x)-g(x)\right)\rightarrow\N\left(0,\Sigma(x)f(x)^{
	-1}\int K^2(u)du\right),
\end{equation}

Lastly, we use the Delta method again to analyze the asymptotic convergence of $\log \widehat \OR(x)$. Recall that $\log \widehat \OR(x)=\log(\hat g^1(x)) + \log(\hat g^2(x)) - \log(\hat g^3(x))- \log(\hat g^4(x))$. By definition and Taylor expansion, 
\begin{align*}
    &\log\widehat \OR(x)- \log \OR(x) \\
    &= \frac{1}{g^1(x)}(\hat g^1(x) - g^1(x)) +  \frac{1}{g^2(x)}(\hat g^2(x) - g^2(x))\\
    &- \frac{1}{g^3(x)}(\hat g^3(x) - g^3(x))- \frac{1}{g^4(x)}(\hat g^4(x) - g^4(x))
\end{align*}

Therefore, by Delta method and Eq~\eqref{eq:asymp-g}, the asymptotic variance of $\log\widehat \OR(x)$ is:
\begin{align*}
    A^\top(x)\Sigma(x)A(x) f(x)^{-1}\int K^2(u)du.
\end{align*}
Therefore, we have
\begin{align*}
   (n_{2}\lambda^d)^{1/2} (\log\widehat \OR(x)- \log \OR(x)) \rightarrow \mathcal{N}\left(0,   A^\top(x)\Sigma(x)A(x) f(x)^{-1}\int K^2(u)du\right),
\end{align*}
in distribution. This completes the proof.
\end{proof}

Let $ \log\widehat \OR_1(x),  \log\widehat \OR_2(x)$ denotes the two consistent estimators obtained from datasets $\OCal_1$ and $\OCal_2$. Let $\widehat{\tau}_{2}$ denote a consistent estimator of the true ATE $\tau$ that we obtain using dataset $\OCal_2$. Then
\begin{align}\label{eq:asymp-kernel}
\begin{split}
	\left(\begin{array}{c}
	n_{2}^{1/2}(\widehat{\tau}_{2}-\tau)\\
	(n_{2}\lambda^d)^{1/2}\left(\log\widehat \OR_1(x)-  \log\widehat \OR_2(x)\right)
	\end{array}\right) 
	\rightarrow\N\left\{ 0,\left(\begin{array}{cc}
		v_{2} & \Gamma^{\T}\\
		\Gamma & V
	\end{array}\right)\right\},
\end{split}
\end{align}

for some $V$ and $\Gamma$. If Eq.~(\ref{eq:asymp-kernel}) holds exactly rather than asymptotically, by multivariate normal theory, we have the following the conditional
distribution:
\begin{align*}
    &n_{2}^{1/2}(\widehat{\tau}_{2}-\tau)\mid 	(n_{2}\lambda^d)^{1/2}\left(\log\widehat \OR_1(x)-  \log\widehat \OR_2(x)\right)\\
    &\sim \mathcal{N}\left\{(n_{2}\lambda^d)^{1/2}\Gamma^{\T}V^{-1}	(n_{2}\lambda^d)^{1/2}\left(\log\widehat \OR_1(x)-  \log\widehat \OR_2(x)\right),v_{2}-\Gamma^{\T}V^{-1}\Gamma\right\}.
\end{align*}

Then, we apply the control variates method to build a new estimator of $\tau$ which has a lower variance than $\hat \tau_2$. The new bias-corrected estimator for ATE is as follows: $\widehat{\tau}_{\CV}(\beta) = \widehat{\tau}_{2}-\beta \left(\log\widehat \OR_1(x)-  \log\widehat \OR_2(x)\right)$.

Solving for the optimal $\beta$, we obtain the new estimator 
\begin{equation}\label{eq:proposed-estimator-kernel}
	\widehat{\tau}_{\CV}=\widehat{\tau}_{2}-\sqrt{\lambda^d}\Gamma^{\T}V^{-1}\left(\log\widehat \OR_1(x)-  \log\widehat \OR_2(x)\right),
\end{equation}

where $V = \text{Var}\left(\log\widehat \OR_1(x)-  \log\widehat \OR_2(x)\right)^{-1}$, and $\Gamma = \text{Cov}(\log\widehat \OR_1(x)-  \log\widehat \OR_2(x), \widehat{\tau}_2)$, and $\lambda^\ast = \littleO(\frac{1}{N})^{\frac{1}{d+4}}$.

Denote the asymptotic variance of $\widehat{\tau}_{2}$ as $v_2$. Under Assumption \ref{assump-ignorable}, if Equation (\ref{eq:asymp-kernel}) holds, then $\widehat{\tau}_{\CV}$ is consistent for $\tau$, and we have:
\begin{equation*}\label{eq:asymp-var-kernel}
	n_{2}^{1/2}(\widehat{\tau}_{\CV}-\tau)\rightarrow\N(0,v_{2}-\Gamma^{\T}V^{-1}\Gamma),
\end{equation*}
in distribution as $n_{2}\rightarrow\infty$. Given a nonzero $\Gamma$, the asymptotic variance, $v_{2}-\Gamma^{\T}V^{-1}\Gamma,$ is smaller than $v_2$.

%% file: sec/app_bootstrap.tex

\section{Bootstrap Sampling Procedure}\label{app:bootstrap}

Our bootstrap sampling procedure is similar to the one in \citet{yang2020combining}. For $b = 1,...,B$, we construct bootstrap replicates for the estimators as follows:

\textbf{Step 1.} Sample $n_2$ units from  $\OCal_2$ with replacement as $O_2^{*(b)}$, and sample $n_1$ units from  $\OCal_1$ with replacement as $O_1^{*(b)}$.

\textbf{Step 2.} Compute the bootstrap replicate $\hat{\tau}_2^{(b)}$ using the dataset $O_2^{*(b)}$, and compute the bootstrap replicates $\hat{\psi}_{2}^{(b)}$, and $\hat{\psi}_{1}^{(b)}$ using the dataset  $O_1^{*(b)}$.

Based on the bootstrap replicates, we estimate the sample covariance $\hat{\Gamma}$ and $\hat{V}$ by
\begin{align*}
    \hat{\Gamma} &= \frac{1}{B - 1} \sum_{b = 1}^B (\hat{\tau}_2^{(b)} - \hat{\tau}_2)(\hat{\psi}_{2}^{(b)} - \hat{\psi}_{1}^{(b)} - \hat{\psi}_{2} + \hat{\psi}_{1}), \\
    \hat{V} &= \frac{1}{B - 1} \sum_{b = 1}^B( \hat{\psi}_{2}^{(b)} - \hat{\psi}_{1}^{(b)} 
    - \hat{\psi}_{2} + \hat{\psi}_{1})
    ( \hat{\psi}_{2}^{(b)} - \hat{\psi}_{1}^{(b)} 
    - \hat{\psi}_{2} + \hat{\psi}_{1})^{\top}.
\end{align*}

The bootstrap covariance estimates $\hat{\Gamma}$ and $\hat{V}$ are consistent if the estimators $\hat{\tau}_2$, $\hat{\psi}_{1}$, and $\hat{\psi}_{2}$ are regular asymptotically linear (RAL) estimators, as shown by \citet{efron1986bootstrap} and \citet{shao2012bootstrap}.

\begin{definition}
An estimator $\hat{\tau}$ for a statistic $\tau$ estimated from a dataset $\{Z_i,X_i,Y_i\}_{i=1}^n$ is RAL if it can be asymptotically approximated by a sum of IID random vectors with mean 0: 
\begin{equation*}
    \hat{\tau} - \tau \cong \frac{1}{n} \sum_{i=1}^n \phi(Z_i, X_i, Y_i)
\end{equation*}
$\phi(Z,X,Y)$ is also known as the influence function for $\hat{\tau}$.
\end{definition}

A common example of a RAL estimator for the ATE is the regression imputation estimator, which we used in experiments.

\subsection{Matching estimators}

Another common class of ATE estimators is matching estimators. Matching estimators do not have smooth influence functions, so the direct bootstrap procedure above may not be consistent \cite{abadie2008bootstrapfailure}. However, \citet{yang2020combining} and \citet{abadie2006largesample} show that the bias of a matching estimator $\hat{\tau}$ can still be expressed in an asymptotically linear form:
\begin{equation*}
    \hat{\tau} - \tau \cong \frac{1}{n} \sum_{i=1}^n \phi_i
\end{equation*}

Using these linear terms, a slightly modified bootstrap procedure can be used, which \citet{yang2020combining} show to be consistent for both RAL estimators and matching estimators. This procedure uses a modified version of Step 2 which estimates the asymptotically linear terms. Let $\phi_i^{\tau_2}$ indicate the asymptotically linear term for estimator $\hat{\tau}_2$ (e.g. $\phi(Z_i, X_i, Y_i)$ for RAL $\hat{\tau}_2$), and let $\phi_i^{\psi_1}$, $\phi_i^{\psi_2}$ indicate the same for estimators $\hat{\psi}_2$, $\hat{\psi}_2$, respectively. Let $\hat{\phi}_i^{\tau_2}$, $\hat{\phi}_i^{\psi_1}$, $\hat{\phi}_i^{\psi_2}$  denote estimates for the population quantities.

\textbf{Step 2 (modified for matching).} Compute the bootstrap replicates using the dataset $O_2^{*(b)}$ as
\begin{equation*}
    \hat{\tau}_2^{(b)} - \hat{\tau}_2 = \frac{1}{n_2} \sum_{i=1}^{n_2} \hat{\phi}_i^{\tau_2}
\end{equation*}
and compute the bootstrap replicates using the dataset  $O_1^{*(b)}$ as
\begin{equation*}
    \hat{\psi}_{1}^{(b)} - \hat{\psi}_1 = \frac{1}{n_1} \sum_{i=1}^{n_1} \hat{\phi}_i^{\psi_1};\quad \hat{\psi}_{2}^{(b)} - \hat{\psi}_2 = \frac{1}{n_1} \sum_{i=1}^{n_1} \hat{\phi}_i^{\psi_2}.
\end{equation*}

\begin{theorem}\label{thm:bootstrap} (Theorem 3 from \citet{yang2020combining})
If $\hat{\tau}_2$, $\hat{\psi}_{1}$, and $\hat{\psi}_{2}$ are RAL estimators or matching estimators, then under certain regularity conditions, the bootstrap estimates $\hat{\Gamma}, \hat{V}$ under this modified procedure are consistent for $\Gamma, V$.
\end{theorem}

%% file: sec/app_proofs.tex
\section{Proofs and Further Analysis of the Odds Ratio} \label{app:proofs}

We provide proofs for the theorems and lemma presented in Sections \ref{sec:methods} and \ref{sec:cv_selection_bias}, as well as further analysis of the odds ratio's effectiveness as a control variate.

\subsection{Proofs from Section \ref{sec:methods}}

\begin{reptheorem}{thm:general-cv-var} Denote the asymptotic variance of $\widehat{\tau}_{2}$ as $v_2$. Under Assumption \ref{assump-ignorable}, if Equation (\ref{eq:asymp}) holds, then $\widehat{\tau}_{\CV}$ is consistent for $\tau$, and we have:
\begin{equation*}
	n_{2}^{1/2}(\widehat{\tau}_{\CV}-\tau)\rightarrow\N(0,v_{2}-\Gamma^{\T}V^{-1}\Gamma),
\end{equation*}
in distribution as $n_{2}\rightarrow\infty$. Given a nonzero $\Gamma$, the asymptotic variance, $v_{2}-\Gamma^{\T}V^{-1}\Gamma,$ is smaller than $v_2$.
\end{reptheorem}

\begin{proof}
The theorem statement follows directly from \begin{align*}
\begin{split}
	n_{2}^{1/2}\left(\begin{array}{c}
		\widehat{\tau}_{2}-\tau\\
		\widehat{\psi}_{2}-\widehat{\psi}_{1}
	\end{array}\right) 
	\rightarrow\N\left\{ 0,\left(\begin{array}{cc}
		v_{2} & \Gamma^{\T}\\
		\Gamma & V
	\end{array}\right)\right\}.
\end{split}
\end{align*}

By construction, $	\widehat{\tau}_{\CV}=\widehat{\tau}_{2}-\Gamma^{\T}V^{-1}(\widehat{\psi}_{2}-\widehat{\psi}_{1})$. To compute the asymptotic variance, notice that 
\begin{align*}
    &\var(n_{2}^{1/2}(\widehat{\tau}_{\CV}-\tau)) \\
    &= \var(n_{2}^{1/2}(\widehat{\tau}_{\CV}-\tau))\\
    &= \var(n_{2}^{1/2}(\widehat{\tau}_{2}-\tau-\Gamma^{\T}V^{-1}(\widehat{\psi}_{2}-\widehat{\psi}_{1}))) \\
    &= \var(n_{2}^{1/2}(\widehat{\tau}_{2}-\tau)) + \Gamma^{\T}V^{-1}\var(n_{2}^{1/2}(\widehat{\psi}_{2}-\widehat{\psi}_{1}))V^{-1}\Gamma - 2\cov(n_{2}^{1/2}(\widehat{\tau}_{2}-\tau),n_{2}^{1/2}\Gamma^{\T}V^{-1}(\widehat{\psi}_{2}-\widehat{\psi}_{1}))\\
    &= v_2 + \Gamma^{\T}V^{-1}\Gamma - 2\Gamma^{\T}V^{-1}\Gamma) \\
    &= v_2 - \Gamma^{\T}V^{-1}\Gamma)
\end{align*}
Therefore, we have:
\begin{equation*}
	n_{2}^{1/2}(\widehat{\tau}_{\CV}-\tau)\rightarrow\N(0,v_{2}-\Gamma^{\T}V^{-1}\Gamma),
\end{equation*}
in distribution as $n_{2}\rightarrow\infty$, which completes the proof.
\end{proof}

\subsection{Proofs from Section \ref{sec:cv_selection_bias}}

\begin{replemma}{lem:or_selection_bias}
If the selection $S$ depends solely on $Y$ (as in Figure \ref{fig:selection_bias_y}), then the conditional odds ratio is transportable and given by:
\begin{align*}
    \begin{split}
         \OR(x) = \frac{P(Y = 1 |S=1, Z=1,  x) P(Y = 0 | S=1,Z=0,  x)}{P(Y = 0 |S=1, Z=1,   x) P(Y = 1 |S=1, Z=0,   x)}.
    \end{split}
\end{align*}
\end{replemma}

\begin{proof}
By Bayes' theorem, 
\begin{equation*}
    P(Y = y|Z=z, x) = \frac{P(Y = y |  S = 1, Z = z, x) P(S = 1 | Z = z, x)}{P(S = 1 | Y = y, Z = z, x)}.
\end{equation*}
Since $S$ depends solely on $Y$, $S$ is conditionally independent of $X$ and $Z$ given $Y$. Therefore, we can rewrite the equation above as 
\begin{equation*}\label{eq:or_bayes}
    P(Y = y|Z=z, x) = \frac{P(Y = y |  S = 1, Z = z, x) P(S = 1 | Z = z, x)}{P(S = 1 | Y = y)}.
\end{equation*}

Substituting this into Definition~\ref{def:odds} (under Assumption~\ref{assump-ignorable}),
\begin{align*}
    \OR(x) &= \frac{P(Y = 1 | Z=1, x) P(Y = 0 | Z=0, x)}{P(Y = 0 | Z=1, x) P(Y = 1 | Z=0, x)} \\
    &= \frac{\frac{P(Y = 1 |  S = 1, Z = 1, x) P(S = 1 | Z = 1, x)}{P(S = 1 | Y = 1)} \frac{P(Y = 0 |  S = 1, Z = 0, x) P(S = 1 | Z = 0, x)}{P(S = 1 | Y = 0)}}{\frac{P(Y = 0 |  S = 1, Z = 1, x) P(S = 1 | Z = 1, x)}{P(S = 1 | Y = 0)} \frac{P(Y = 1 |  S = 1, Z = 0, x) P(S = 1 | Z = 0, x)}{P(S = 1 | Y = 1)}} \\
    &= \frac{P(Y = 1 |  S = 1, Z = 1, x) P(Y = 0 |  S = 1, Z = 0, x) }{P(Y = 0 |  S = 1, Z = 1, x)  P(Y = 1 |  S = 1, Z = 0, x) }
\end{align*}
Therefore, the conditional odds ratio is transportable under selection bias given by $S$.
\end{proof}

\begin{reptheorem}{thm:OR-robust-selection-bias}
If the selection $S$ depends solely on $Y$ (as in Figure \ref{fig:selection_bias_y}) and $P(Y=1 |  Z = z, X = x)$ follows the logistic model in \eqref{eq:varying-coeff}, then $P(Y=1 | Z = z,X = x,  S = 1)$ also follows a logistic model, with the same coefficient $\beta_1^x$ on $Z$ as the logistic model for $P(Y=1 |  Z = z, X = x)$ for each covariate value $x$. Furthermore, the conditional odds ratio $\OR(x) = e^{\beta_1^x}$.
\end{reptheorem}

\begin{proof}
Given the assumed outcome model, we have
\begin{equation*}
    P(Y=1 | Z = z, x) = \frac{e^{\beta_0^x + \beta_1^x z}}{1 + e^{\beta_0^x + \beta_1^x z}}.
\end{equation*}

Let $p_1 = P(S = 1 | Y = 1)$ and $p_0 = P(S=1 | Y = 0)$. Since the selection $S$ depends solely on $Y$, $S$ is conditionally independent of $X$ and $Z$ given $Y$.

The outcome model under selection bias is given by:
\begin{align*}
    &P(Y=1 | Z = z, X = x, S = 1) \\
    &= \frac{P(Y=1 | Z = z, X = x)P(S = 1 | Z = z, X = x, Y = 1)}{P(S = 1 | Z = z, X = x)} \\
    &= \frac{P(Y=1 | Z = z, X = x)P(S = 1 | Z = z, X = x,  Y = 1)}{\sum_{y \in \{0,1\}}P(Y=y | Z = z, X = x, )P(S = 1 | Z = z, X = x,  Y = y)} \\
    &= \frac{P(Y=1 |Z = z, X = x)p_1}{P(Y=1 | Z = z, X = x)p_1 + P(Y=0 | Z = z, X = x)p_0} \\
    &= \frac{\frac{e^{\beta_0^x + \beta_1^x z}}{1 + e^{\beta_0^x + \beta_1^x z}}p_1}{\frac{e^{\beta_0^x + \beta_1^x z}}{1 + e^{\beta_0^x + \beta_1^x z}}p_1 + \frac{1}{1 + e^{\beta_0^x + \beta_1^x z}}p_0} \\
    &= \frac{e^{\beta_0^x + \beta_1^x z}p_1}{e^{\beta_0^x + \beta_1^x z}p_1 + p_0} \\
    &= \frac{e^{\beta_0^x + \beta_1^x z}p_1/p_0}{e^{\beta_0^x + \beta_1^x z}p_1/p_0 + 1} \\
    &= \frac{e^{\delta + \beta_0^x + \beta_1^x z}}{1 + e^{\delta + \beta_0^x + \beta_1^x z}}.
\end{align*}

where $\delta = \log(p_1/p_0)$. Thus, on the selection biased dataset $\OCal_1$, the outcome model $P(Y=1 | Z = z,X = x,  S = 1)$ also follows a logistic model, with the same coefficient $\beta_1^x$ on $Z$ as the logistic model for $P(Y=1 |  Z = z, X = x)$ for each covariate value $x$. Furthermore, a simple calculation shows that the conditional odds ratio is $\OR(x) = e^{\beta_1^x}$:
\begin{align*}
     \OR(x) &= \frac{P(Y = 1 | Z=1, x) P(Y = 0 | Z=0, x)}{P(Y = 0 | Z=1, x) P(Y = 1 | Z=0, x)} \\
     &= \frac{\frac{e^{\beta_0^x + \beta_1^x}}{1 + e^{ \beta_0^x + \beta_1^x}} \frac{1}{1 + e^{\beta_0^x}}}{\frac{1}{1 + e^{\beta_0^x + \beta_1^x}} \frac{e^{\beta_0^x}}{1 + e^{ \beta_0^x}}} \\
     &= \frac{e^{\beta_0^x + \beta_1^x}}{e^{\beta_0^x}} \\
     &= e^{\beta_1^x}.
\end{align*}
\end{proof}

\subsection{Analysis of Nonlinear Relationship Between ATE and OR}

As discussed in Section \ref{sec:methods}, the variance reduction from adding control variates depends on the strength of the correlation between the control variates and the ATE estimator. Since we propose to use the odds ratio for selection biased datasets, here we examine the relationship between the ATE and the odds ratio. Specifically, we derive an explicit expression for the ATE using the marginal odds ratio $\OR$ assuming a binary covariate $X$ and the following simple logistic outcome model:

\begin{equation*}
    P(Y = 1 | Z = z, X = x) = \frac{e^{\beta_0 + \beta_1 z + \beta_2 x}}{1 + e^{\beta_0 + \beta_1 z + \beta_2 x} },
\end{equation*}

where the marginal odds ratio $\OR$ is defined as
\begin{equation*}
    \OR = \frac{P(Y(1) = 1) P(Y(0) = 0)}{P(Y(1) = 0) P(Y(0) = 1)}.
\end{equation*}

Under this simple logistic outcome model, $\OR = e^{\beta_1}$.

By Assumption~\ref{assump-ignorable}, we have
\begin{align*}
    \E[Y(1)] 
    &= \int \E[Y|Z=1, X=x] P(X=x) \text{d} x\\
    &= \int \frac{e^{\beta_0 + \beta_1 + \beta_2x}}{e^{\beta_0 + \beta_1 + \beta_2x+1}}P(X=x) \text{d} x \\
    &= \frac{e^{\beta_0 + \beta_1 + \beta_2}}{e^{\beta_0 + \beta_1 + \beta_2+1}}P(X=1) + \frac{e^{\beta_0 + \beta_1}}{e^{\beta_0 + \beta_1 +1}}(1-P(X=1))
\end{align*}
Similarly, 
\begin{align*}
    \E[Y(0)] 
    &= \int \E[Y|Z=0, X=x] P(X=x) \text{d} x\\
    &= \int \frac{e^{\beta_0  + \beta_2x}}{e^{\beta_0  + \beta_2x+1}}P(X=x) \text{d} x \\
    &= \frac{e^{\beta_0  + \beta_2}}{e^{\beta_0 + \beta_2+1}}P(X=1) + \frac{e^{\beta_0 }}{e^{\beta_0  +1}}(1-P(X=1))
\end{align*}

Therefore, by some algebra we obtain that 
\begin{align*}
   &\tau = \E[Y(1)]-\E[Y(0)]\\
   &= \frac{\gamma ab\psi}{ab\psi+1} - \frac{\gamma a\psi}{a\psi+1} + \frac{a\psi-a}{a^2\psi-a\psi+a+1} - C,
\end{align*}
where $\psi = \OR$, $\gamma = P(X=1)$, $a = e^{\beta_0}$, $b = e^{\beta_2}$, with a constant term $C = \frac{a}{a+1} - \frac{ab}{ab+1}$.

%% file: sec/app_exp_sim.tex

\section{Additional Experimental Details and Results for Simulation Study}\label{app:experiments_sim}

This section provides additional experimental details and results for the simulation study. All code for running the simulation study is provided with the supplementary materials.

\subsection{Data generation}\label{app:experiments_sim_data}

We generate the dataset $\OCal_2$ by sampling $n_2$ samples using the following data-generating process. Let $X \in \mathbb{R}^2$ have two components $X_1, X_2$, which are i.i.d.\ Bernoulli($p=0.5$). Given $X$, the treatment assignment $Z$ is distributed as 
$P(Z = 1 | X = x) = \frac{e^{a_0 + a_1^Tx}}{1 + e^{a_0 + a_1^Tx}}.$
    
As specific parameters, we set $a_1 = [-1, 1]$, and $a_0 = -E[a_1^T X]$, which implies that $P(Z = 1) = 0.5$. Setting $a_1 = [0, 0]$ would correspond to a randomized study, whereas we set $a_1 = [-1, 1]$ to simulate an observational study with confounding. 
The potential outcomes are distributed as 
\begin{align}
\begin{split}
    P(Y(0) = 1 |x) = \frac{e^{b_{0,0} + b_{0,1}^Tx}}{1 + e^{b_{0,0} + b_{0,1}^Tx}}, \quad P(Y(1) = 1 |x) = \frac{e^{b_{1,0} + b_{1,1}^Tx}}{1 + e^{b_{1,0} + b_{1,1}^Tx}}.
\end{split}\label{eq:sim_potential_outcomes}
\end{align}
Eq.~\eqref{eq:sim_potential_outcomes} builds in ignorability in Assumption \ref{assump-ignorable}, which is also equivalent to generating the outcome $Y$ from
\begin{equation*}
    P(Y = 1 | Z = z, x) = \frac{e^{\beta_0 + \beta_1 z + \beta_2^T x + \beta_3^T xz}}{1 + e^{\beta_0 + \beta_1 z + \beta_2^T x + \beta_3^T xz}},
\end{equation*}
where $\beta_0 = b_{0,0}$, $\beta_1 = b_{1,0} - b_{0,1}$, $\beta_2 = b_{0,1}$, and $\beta_3 = b_{1,1} - b_{0,1}$.
When $\beta_3 = 0$ and $b_{1,1} = b_{0,1}$, then there is no interaction term between $X$ and $Z$ and the conditional odds ratio is simply $e^{\beta_1}$. We set $b_{0,1} = [-1,1]$ and $b_{1,1} = [1, -1]$ ($\beta_3 \neq 0$) so that the conditional odds ratio varies as a function of $x$. 
As done by \citet{zhang2009estimatingoddsratio}, the intercept terms are determined by $b_{0,0} = -0.5 - E[b_{0,1}^T X]$ and $b_{1,0} = 0.5 - E[b_{1,1}^T X]$.

\subsection{Finite-sample experiment scenarios}\label{app:experiments_sim_scenarios}

We consider three scenarios to analyze the finite-sample performance of the proposed estimators. For all three scenarios, we compare the variance for the ATE estimator $\widehat{{\tau}}_{2}$ with the variance of the ATE estimator with control variates $\widehat{{\tau}}_{CV}$. To compare these, we ran $B=100$ bootstrap replicates to measure the variance of the ATE estimator with and without the odds ratio control variates. 

\textbf{Scenario 1:} We vary the size of the observational dataset, $n_2$, while keeping a constant ratio for the size of the observational dataset relative to the size of the selection biased dataset: $n_2/n_1 = 1/10$. This illustrates the simple asymptotic performance of the estimators as the sample sizes increase without changing the proportional sizes of the two datasets relative to each other. 

\textbf{Scenario 2:} We vary the size of the observational dataset, $n_2$, while keeping the size of the selection biased dataset constant and relatively large: $n_1 = 10000$. This illustrates the scenario when a practitioner has access to a large fixed amount of case-control data with selection bias ($\OCal_1$), and must decide how much observational or experimental data to collect with identifiable ATE ($\OCal_2$).

\textbf{Scenario 3:} We vary the size of the selection bias dataset, $n_1$, while keeping the size of the observational dataset constant and relatively small: $n_2 = 1000$. This illustrates the relative utility of including more selection biased samples to estimate the control variate. 
While the ratio $n_2/n_1 = 1/10$ is fixed in Scenario 1, in Scenario 3 we consider the effect of varying that ratio for a fixed observational dataset size, $n_2$.

\subsection{Simple logistic outcome model without interaction between \texorpdfstring{$X$}{X} and \texorpdfstring{$Z$}{Z}}\label{app:experiments_sim_simple}

In addition to the data generation setting described in Section \ref{sec:experiments_sim}, we also include results with a simpler data generation setting without interaction between $X$ and $Z$. We set $b_{1,1} = b_{0,1} = (-1,1)^\top$, which implies that $\beta_3 = 0$ in Section \ref{app:experiments_sim_data}. In this simpler model, the conditional odds ratio is constant in $X$ and is given by $e^{\beta_1}$.

Figure \ref{fig:var_reg_simple_logistic} shows that adding control variates still improves the variance of the ATE estimator under this simpler outcome model without interaction between $X$ and $Z$.

\begin{figure*}[!ht]
\centering
\begin{tabular}{ccc} 
\includegraphics[width=0.31\textwidth]{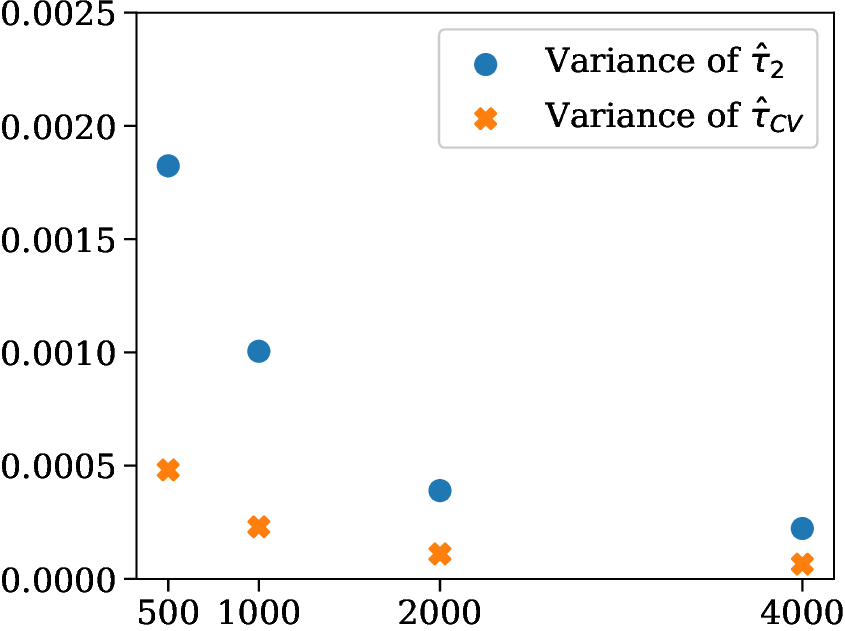} & 
\includegraphics[width=0.31\textwidth]{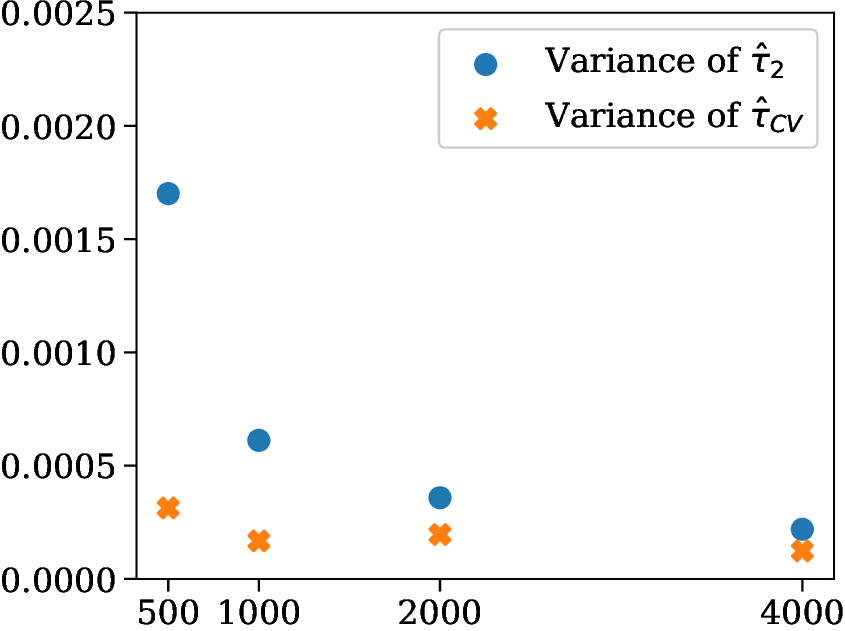} & 
\includegraphics[width=0.31\textwidth]{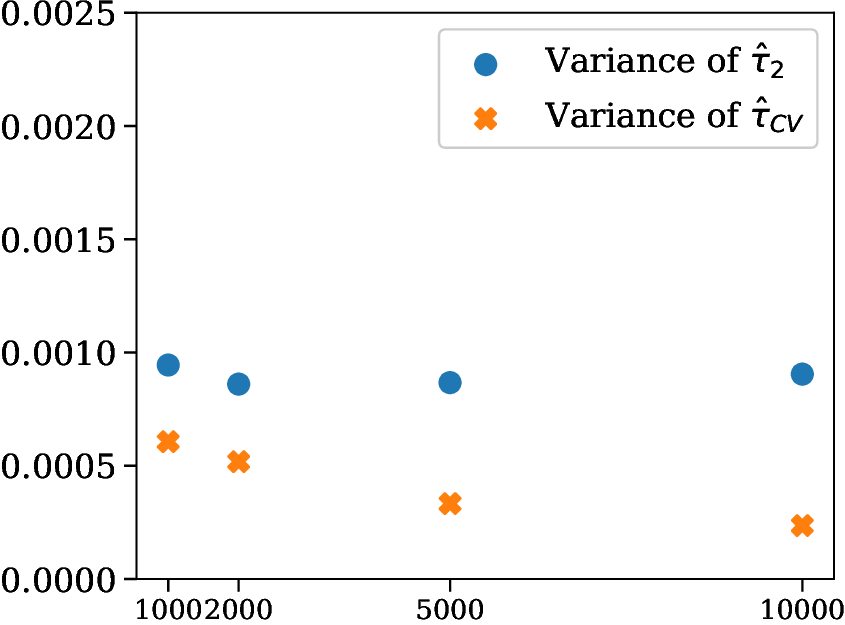} \\
\scriptsize{$n_2$} & \scriptsize{$n_2$} & \scriptsize{$n_2$} \\
\end{tabular}
\caption{\footnotesize{\textbf{Scenario 1 (left):} Comparisons of variance as the size of the observational dataset $n_2$ increases under the simple logistic model with $\beta_3 = 0$. The ratio $n_2/n_1$ is kept constant at $1/10$. \textbf{Scenario 2 (middle):} Comparisons of variance as the size of the observational dataset $n_2$ increases under the simple logistic model with $\beta_3 = 0$. The size of the selection bias dataset is fixed at $n_1=10000$. \textbf{Scenario 3 (right):} Comparisons of variance as the size of the selection bias dataset $n_1$ increases under the simple logistic model with $\beta_3 = 0$. The size of the observational dataset is fixed at $n_2=1000$. \textit{Lower is better.}}}\label{fig:var_reg_simple_logistic}  
\end{figure*}

\subsection{Bias results}\label{app:experiments_sim_bias}

In this section, we report the bias for the ATE estimators with and without control variates,i.e. $\widehat{\tau}_2$ and $\widehat{\tau}_{CV}$. The bias is calculated over $B=100$ bootstrap replicates, and is defined as the difference between the average value of the ATE estimator over $B=100$ bootstrap replicates and the true ATE. 

The bias for the simple logistic outcome model without interaction between $X$ and $Z$ described in Appendix \ref{app:experiments_sim_simple} is given in Figures \ref{fig:bias_reg_ratio-fixed_simple},  \ref{fig:bias_reg_n1-fixed_simple}, and \ref{fig:bias_reg_n2-fixed_simple}. The bias for the logistic model with varying coefficients described in Section \ref{sec:experiments_sim} and Appendix \ref{app:experiments_sim_data} is given in Figures \ref{fig:bias_reg_ratio-fixed_varying-coef},  \ref{fig:bias_reg_n1-fixed_varying-coef}, and \ref{fig:bias_reg_n2-fixed_varying-coef}. In general, across all three scenarios, the bias decreases as $n_2$ increases, and is not significantly different with or without control variates. We report the bias mainly to confirm that adding control variates does not significantly increase the bias of the estimator in finite samples. 

\begin{figure}[!ht]
    \centering
    \begin{tabular}{c}
    Simple logistic model ($\beta_3 = 0$) \\
    \includegraphics[width=0.4\textwidth]{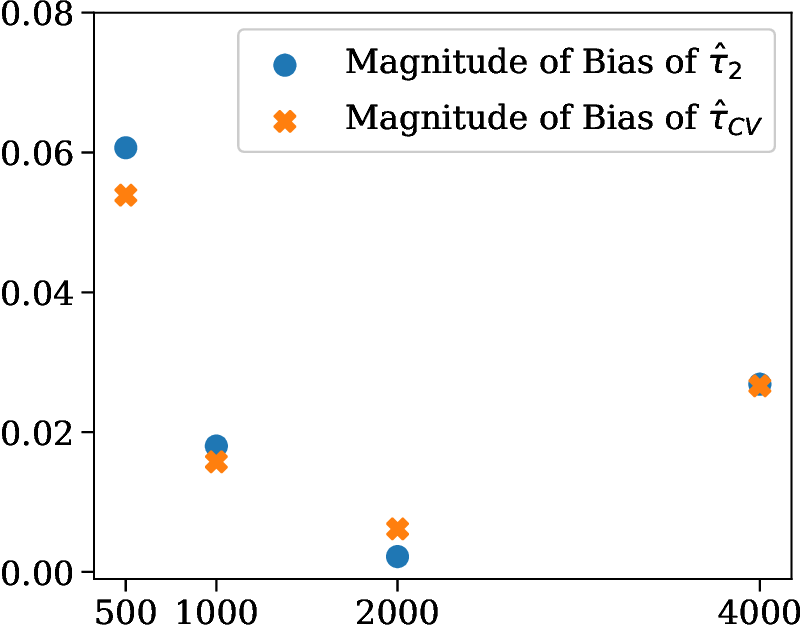} \\
    $n_2$ 
    \end{tabular}
    \caption{\textbf{Scenario 1:} Comparisons of bias as the size of the observational dataset $n_2$ increases under the simple logistic model. The ratio $n_2/n_1$ is kept constant at $1/10$.}
    \label{fig:bias_reg_ratio-fixed_simple}
\end{figure}

\begin{figure}[!ht]
    \centering
    \begin{tabular}{c}
    Simple logistic model ($\beta_3 = 0$) \\
    \includegraphics[width=0.4\textwidth]{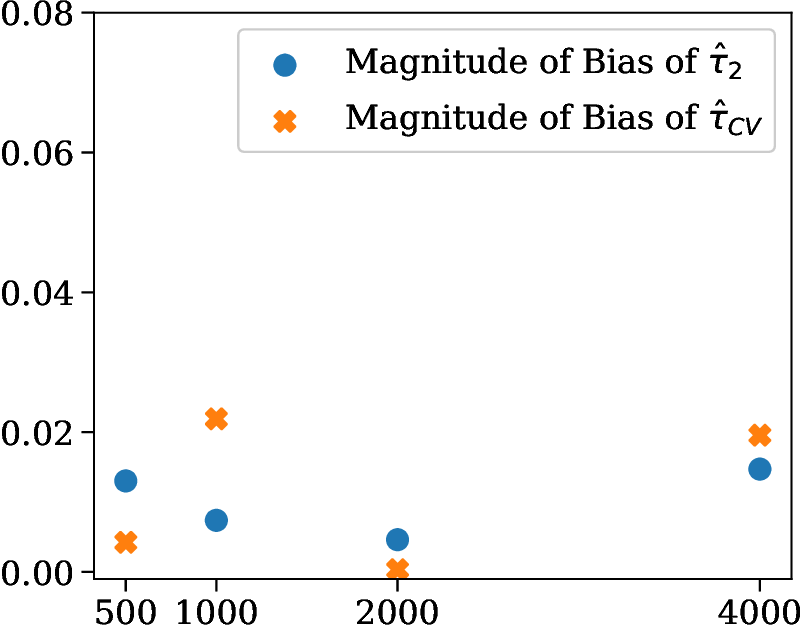}  \\
    $n_2$ 
    \end{tabular}
    \caption{\textbf{Scenario 2:} Comparisons of bias as the size of the observational dataset $n_2$ increases under the simple logistic model. The size of the selection bias dataset is fixed at $n_1=10000$.}
    \label{fig:bias_reg_n1-fixed_simple}
\end{figure}

\begin{figure}[!ht]
    \centering
    \begin{tabular}{c}
    Simple logistic model ($\beta_3 = 0$) \\
    \includegraphics[width=0.4\textwidth]{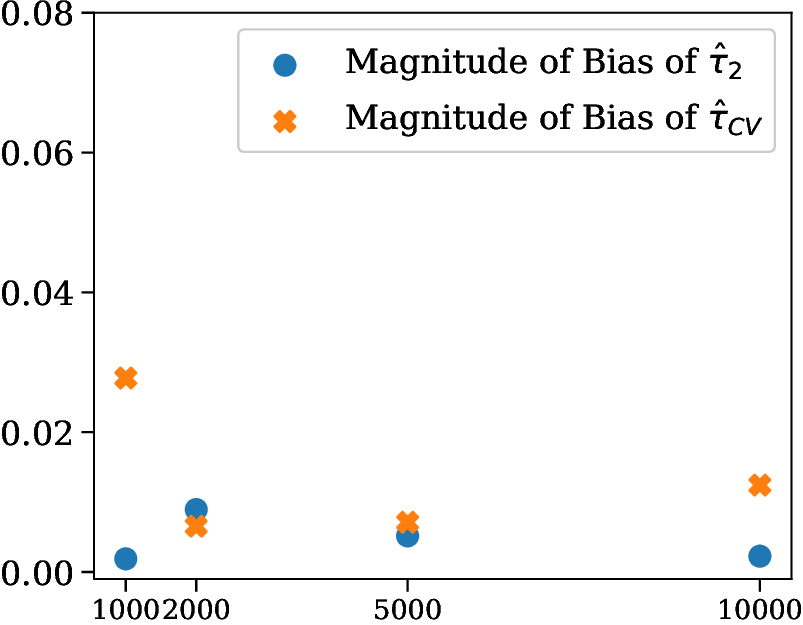} \\
    $n_1$ 
    \end{tabular}
    \caption{\textbf{Scenario 3:} Comparisons of bias as the size of the selection bias dataset $n_1$ increases under the simple logistic model. The size of the observational dataset is fixed at $n_2=1000$.}
    \label{fig:bias_reg_n2-fixed_simple}
\end{figure}

\begin{figure}[!ht]
    \centering
    \begin{tabular}{c}
    Varying coefficient model ($\beta_3 \neq 0$)\\
    \includegraphics[width=0.4\textwidth]{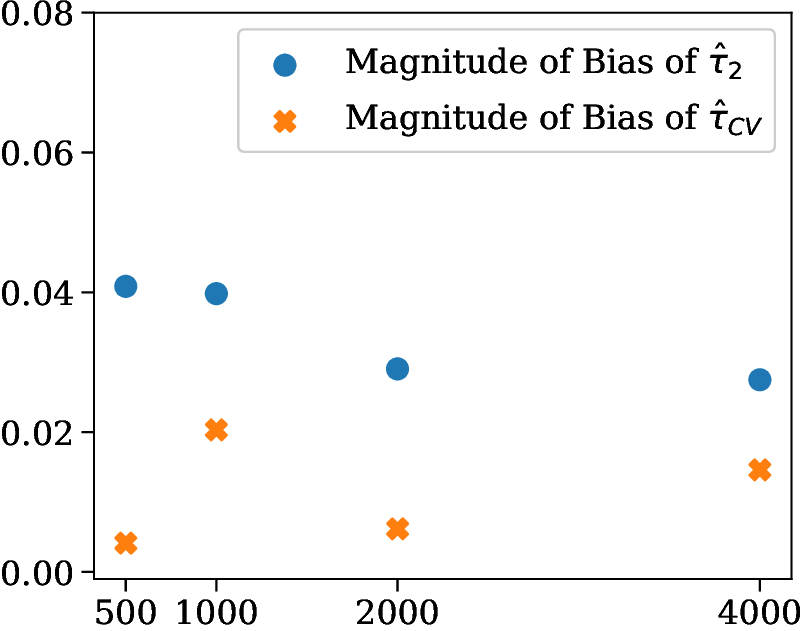} \\
    $n_2$
    \end{tabular}
    \caption{\textbf{Scenario 1:} Comparisons of bias as the size of the observational dataset $n_2$ increases under the logistic model with varying coefficients ($\beta_3 \neq 0$) from Section \ref{sec:experiments_sim} and Appendix \ref{app:experiments_sim_data}. The ratio $n_2/n_1$ is kept constant at $1/10$.}
    \label{fig:bias_reg_ratio-fixed_varying-coef}
\end{figure}

\begin{figure}[!ht]
    \centering
    \begin{tabular}{c}
    Varying coefficient model ($\beta_3 \neq 0$)\\
    \includegraphics[width=0.4\textwidth]{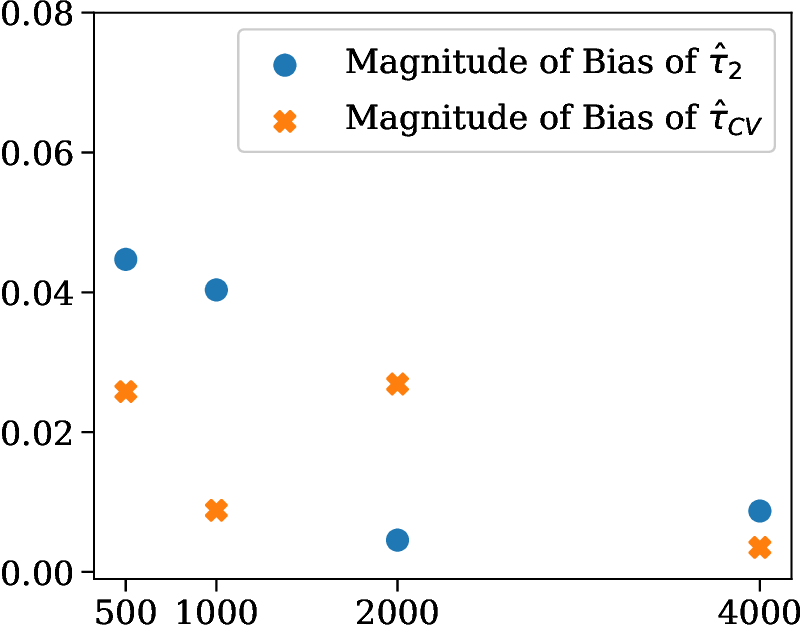} \\
    $n_2$ 
    \end{tabular}
    \caption{\textbf{Scenario 2:} Comparisons of bias as the size of the observational dataset $n_2$ increases under the logistic model with varying coefficients ($\beta_3 \neq 0$) from Section \ref{sec:experiments_sim} and Appendix \ref{app:experiments_sim_data}. The size of the selection bias dataset is fixed at $n_1=10000$.}
    \label{fig:bias_reg_n1-fixed_varying-coef}
\end{figure}

\begin{figure}[!ht]
    \centering
    \begin{tabular}{c}
    Varying coefficient model ($\beta_3 \neq 0$)\\
    \includegraphics[width=0.4\textwidth]{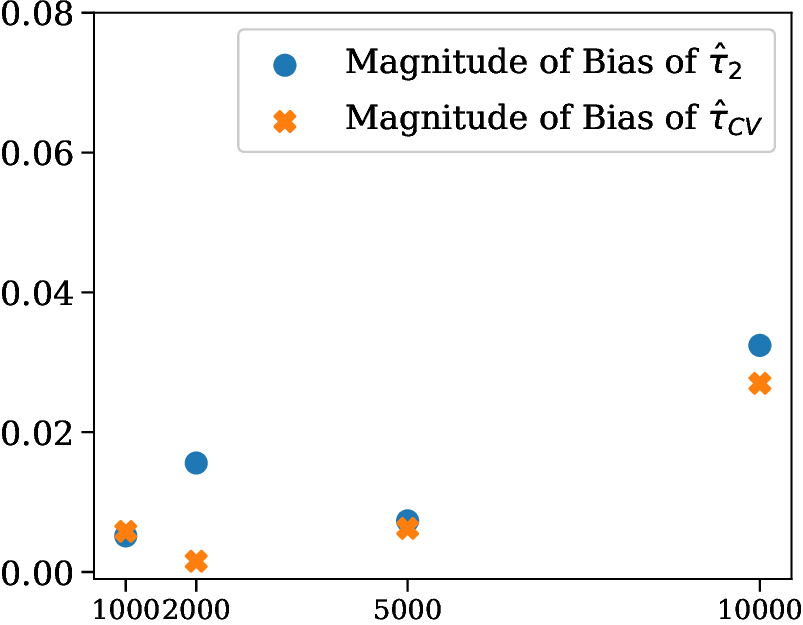} \\
    $n_1$ 
    \end{tabular}
    \caption{\textbf{Scenario 3:} Comparisons of bias as the size of the selection bias dataset $n_1$ increases under the logistic model with varying coefficients ($\beta_3 \neq 0$) from Section \ref{sec:experiments_sim} and Appendix \ref{app:experiments_sim_data}. The size of the observational dataset is fixed at $n_2=1000$.}
    \label{fig:bias_reg_n2-fixed_varying-coef}
\end{figure}

%% file: sec/app_exp_real.tex
\section{Additional Experimental Details and Results for Real Data Case Studies}\label{app:exp_details}

This section provides additional experimental details and results for the real data case studies. All code for generating data and running the experiments is provided with the supplementary materials.

\subsection{Data generation for case study 1: flu shot encouragement}

We provide a detailed breakdown for generating the selection biased dataset $\mathcal{O}_1$ for Case Study 1 on flu shot encouragement below. Code for this is also included with the supplementary materials.
\begin{enumerate}
    \item Fit a logistic regression model on the original dataset $\mathcal{O}_2$ with inputs $X,Z$ and outcome $Y$ according to \begin{equation}
    P(Y=1 | Z = z, X = x) = g(\beta, z, x) = \frac{e^{\beta_0 + \beta_1 z + \beta_2 x + \beta_3xz}}{1 + e^{\beta_0 + \beta_1 z + \beta_2 x + \beta_3xz}}.
    \end{equation}
    This results in estimated parameters $\hat{\beta}$.
    \item Fit a logistic regression model to estimate the propensity score according to \begin{equation*}
    P(Z = 1 | X = x) = h(a, x) = \frac{e^{a_0 + a_1^Tx}}{1 + e^{a_0 + a_1^Tx}},
    \end{equation*}
    This results in estimated parameters $\hat{a}$.
    \item Sample covariates $\{X_i\}_1^N$ with replacement from $\mathcal{O}_2$.
    \item Generate $Z_i$ according to the estimated propensity score distribution, $P(Z_i = 1 | X = X_i) = h(\hat{a}, X_i)$.
    \item Generate $Y_i$ according to the estimated outcome distribution, $P(Y_i=1 | Z = Z_i, X = X_i) = g(\hat{\beta}, Z_i, X_i)$.
    \item Apply selection bias on the outcome according to the distribution $P(S_i = 1 | Y_i = 1) = 0.9$ to generate the final dataset $\mathcal{O}_1$.
\end{enumerate}

\subsection{Data generation for case study 2: spam email detection}

For case study 2, we directly apply the code provided by the Atlantic Causal Inference Conference (ACIC) Data Challenge to generate both $\OCal_1$ and $\OCal_2$. The ACIC data generation code is publicly available at \url{https://sites.google.com/view/acic2019datachallenge/data-challenge}. The ACIC data generation code modifies existing real datasets in a variety of ways to generate datasets for the ACIC data challenge with known true ATEs. Specifically, we use ACIC's ``modification 1'' of the Spambase spam email detection dataset from UCI, which applies a logistic outcome model very close to the logistic regression models fitted to the actual data. For convenience, we include the exact script for ``modification 1'' in our supplementary materials.

\subsection{Further details on the estimators}

We provide further details on the estimators for the ATE and odds ratio used in the real data experiments.

\textbf{Logistic model with interaction:} We model the outcome $Y$ using a logistic model with an interaction term between $X$ and $Z$ in Eq.~\eqref{eq:logistic_interaction}, repeated here for convenience:
$$P(Y=1 | Z = z, X = x) = g(\beta, z, x) = \frac{e^{\beta_0 + \beta_1 z + \beta_2 x + \beta_3xz}}{1 + e^{\beta_0 + \beta_1 z + \beta_2 x + \beta_3xz}}.$$

To estimate the ATE from the dataset without selection bias $\OCal_2$, we perform logistic regression of $Y$ on $X$, $Z$, and the interaction term $XZ$, to produce estimates $\hat{\beta}_{\OCal_2}$. The ATE estimator is then given by 
$$\widehat{\tau}_2 = g(\hat{\beta}_{\OCal_2}, 1, x) - g(\hat{\beta}_{\OCal_2}, 0, x). $$

We then estimate the conditional odds ratio from $\OCal_2$ as $\widehat{\OR}_2(x) = e^{\hat{\beta}_{1,\OCal_2} + \hat{\beta}_{3,\OCal_2}}$.

We similarly estimate the conditional odds ratio from $\OCal_1$ by following the same logistic regression procedure as above, resulting in $\widehat{\OR}_1(x) = e^{\hat{\beta}_{1,\OCal_1} + \hat{\beta}_{3,\OCal_1}}$.

\textbf{Neural network:} For a more general outcome model, we model the outcome $Y$ using a logistic model with varying coefficients in Eq.~\eqref{eq:varying-coeff}, where the functions $\beta_0^x = f_0(x; \theta)$, $\beta_1^x = f_1(x; \theta)$ make up the two-dimensional output of a single neural network with parameters $\theta$:

$$P(Y=1 | Z = z, X = x) = g(\theta, z, x) = \frac{e^{f_0(x; \theta) +  f_1(x; \theta) z}}{1 + e^{f_0(x; \theta) +  f_1(x; \theta) z}}.$$

The optimization objective for the neural network is the logistic loss on the final outcome prediction, $g(\theta, z, x)$. We optimize the neural network using using ADAM with a default learning rate of 0.001 for $1,000$ epochs with batch size $n_2$. We choose the neural network architecture using five-fold cross validation over $\OCal_2$. Specifically, we search over $\{4, 8\}$ hidden layers, and hidden layer sizes of $\{4, 8, 16, 32\}$. We use the TensorFlow framework and include all code in the supplementary materials.

Once an architecture has been chosen by the method above, we estimate the ATE from $\OCal_2$ by optimizing the neural network parameters $\theta$ over the dataset $\OCal_2$ to obtain an estimate $\hat{\theta}_{\OCal_2}$, and the ATE estimator is given by 
$$\widehat{\tau}_2 = g(\hat{\theta}_{\OCal_2}, 1, x) - g(\hat{\theta}_{\OCal_2}, 0, x).$$

We then estimate the conditional odds ratio from $\OCal_2$ as $\widehat{\OR}_2(x) = e^{f_1(x; \hat{\theta}_{\OCal_2})}$.

Finally, using the same architecture as chosen above, we estimate the conditional odds ratio from $\OCal_1$ by optimizing the neural network parameters $\theta$ over the dataset $\OCal_1$ to obtain an estimate $\hat{\theta}_{\OCal_1}$, resulting in $\widehat{\OR}_1(x) = e^{f_1(x; \hat{\theta}_{\OCal_1})}$.

\subsection{Bias results}

We report the bias for the ATE estimators with and without control variates, $\widehat{\tau}_2, \widehat{\tau}_{CV}$ for the real data experiments. The bias is calculated over $B=300$ bootstrap replicates, and is defined as the difference between the average value of the ATE estimator over $B=300$ bootstrap replicates and the true ATE.

Tables \ref{tab:flu_bias} and \ref{tab:acic_bias} show the bias of $\widehat{\tau}_2, \widehat{\tau}_{CV}$ for each of the different estimation methods for case study 1 and case study 2, respectively. While there was not much difference in bias between $\widehat{\tau}_2$ and $\widehat{\tau}_{CV}$ in the simulation study (Appendix \ref{app:experiments_sim_bias}), we often observe higher bias for $\widehat{\tau}_{CV}$ in finite samples on the real data. This bias is more pronounced in case study 2, which may be due to a high dimensional continuous covariate vector $X$. 

\begin{table}[!ht]
\caption{\textbf{Case study 1:} Bias on flu shot encouragement data with $n_1 = 10,000$}
\label{tab:flu_bias}
\vskip 0.15in
\begin{center}
\begin{small}
\begin{sc}
\begin{tabular}{lccc}
\toprule
Metric & Logistic & Kernel \\
\midrule
Bias $\widehat{\tau}_2$ & \num{4.283E-04} & \num{5.384E-04} \\
Bias $\widehat{\tau}_{CV}$ & \num{6.431E-04} & \num{-4.464E-03} \\
\bottomrule
\end{tabular}
\end{sc}
\end{small}
\end{center}
\vskip -0.1in
\end{table}

\begin{table}[!ht]
\caption{\textbf{Case study 2:} Bias on spam email detection data with $n_1 = 30,000$.}
\label{tab:acic_bias}
\vskip 0.15in
\begin{center}
\begin{small}
\begin{sc}
\begin{tabular}{lccc}
\toprule
Metric & Logistic & NN & Kernel \\
\midrule
& \multicolumn{3}{c}{$n_2 = 3,000$} \\
\cmidrule{2-4}
Bias $\widehat{\tau}_2$ & \num{6.785E-04} & \num{-1.095E-02} & \num{-4.250E-03}\\
Bias $\widehat{\tau}_{CV}$ & \num{2.132E-03} & \num{-1.129E-02} & \num{-6.390E-03} \\
& \multicolumn{3}{c}{$n_2 = 10,000$} \\
\cmidrule{2-4}
Bias $\widehat{\tau}_2$ & \num{7.230E-03} & \num{3.427E-04} & \num{-2.827E-04}\\
Bias $\widehat{\tau}_{CV}$ & \num{3.085E-03} & \num{1.361E-03} & \num{-1.378E-03} \\
\bottomrule
\end{tabular}
\end{sc}
\end{small}
\end{center}
\vskip -0.1in
\end{table}

%% file: Arxiv_submission.bbl
\begin{thebibliography}{55}
\providecommand{\natexlab}[1]{#1}
\providecommand{\url}[1]{\texttt{#1}}
\expandafter\ifx\csname urlstyle\endcsname\relax
  \providecommand{\doi}[1]{doi: #1}\else
  \providecommand{\doi}{doi: \begingroup \urlstyle{rm}\Url}\fi

\bibitem[Abadie and Imbens(2006)]{abadie2006largesample}
A.~Abadie and G.~W. Imbens.
\newblock Large sample properties of matching estimators for average treatment
  effects.
\newblock \emph{Econometrica}, 74:\penalty0 235–267, 2006.

\bibitem[Abadie and Imbens(2008)]{abadie2008bootstrapfailure}
A.~Abadie and G.~W. Imbens.
\newblock On the failure of the bootstrap for matching estimators.
\newblock \emph{Econometrica}, 76:\penalty0 1537–1557, 2008.

\bibitem[Agarwal and Duchi(2012)]{agarwal2012distributed}
Alekh Agarwal and John~C Duchi.
\newblock Distributed delayed stochastic optimization.
\newblock In \emph{IEEE Conference on Decision and Control (CDC)}, pages
  5451--5452. IEEE, 2012.

\bibitem[Agresti(2015)]{agresti2015foundations}
A.~Agresti.
\newblock \emph{Foundations of Linear and Generalized Linear Models}.
\newblock John Wiley \& Sons, 2015.

\bibitem[Angrist et~al.(1996)Angrist, Imbens, and
  Rubin]{angrist1996identification}
Joshua~D Angrist, Guido~W Imbens, and Donald~B Rubin.
\newblock Identification of causal effects using instrumental variables.
\newblock \emph{Journal of the American Statistical Association}, 91\penalty0
  (434):\penalty0 444--455, 1996.

\bibitem[Bareinboim and Pearl(2016)]{bareinboim2016datafusion}
E.~Bareinboim and J.~Pearl.
\newblock Causal inference and the data-fusion problem.
\newblock \emph{National Academy of Sciences}, 113:\penalty0 7345–7352, 2016.

\bibitem[Bareinboim and Pearl(2012)]{bareinboim2012controlling}
Elias Bareinboim and Judea Pearl.
\newblock Controlling selection bias in causal inference.
\newblock In \emph{Artificial Intelligence and Statistics}, pages 100--108,
  2012.

\bibitem[Bareinboim and Pearl(2014)]{bareinboim2014transportability}
Elias Bareinboim and Judea Pearl.
\newblock Transportability from multiple environments with limited experiments:
  Completeness results.
\newblock \emph{Advances in neural information processing systems},
  27:\penalty0 280--288, 2014.

\bibitem[Bottou et~al.(2018)Bottou, Curtis, and
  Nocedal]{bottou2018optimization}
L{\'e}on Bottou, Frank~E Curtis, and Jorge Nocedal.
\newblock Optimization methods for large-scale machine learning.
\newblock \emph{SIAM Review}, 60\penalty0 (2):\penalty0 223--311, 2018.

\bibitem[Cannings and Fan(2019)]{cannings2019correlation}
Timothy~I Cannings and Yingying Fan.
\newblock The correlation-assisted missing data estimator.
\newblock \emph{arXiv preprint arXiv:1911.01859}, 2019.

\bibitem[Cleveland(1991)]{cleveland1991local}
William~S Cleveland.
\newblock Local regression models.
\newblock \emph{Statistical models in S.}, 1991.

\bibitem[Colnet et~al.(2020)Colnet, Mayer, Chen, Dieng, Li, Varoquaux, Vert,
  Josse, and Yang]{colnet2020causal}
B{\'e}n{\'e}dicte Colnet, Imke Mayer, Guanhua Chen, Awa Dieng, Ruohong Li,
  Ga{\"e}l Varoquaux, Jean-Philippe Vert, Julie Josse, and Shu Yang.
\newblock Causal inference methods for combining randomized trials and
  observational studies: a review.
\newblock \emph{arXiv preprint arXiv:2011.08047}, 2020.

\bibitem[Didelez et~al.(2010)Didelez, Kreiner, and
  Keiding]{didelez2010graphical}
V.~Didelez, S.~Kreiner, and N.~Keiding.
\newblock Graphical models for inference under outcome-dependent sampling.
\newblock \emph{Statistical Science}, 25\penalty0 (3):\penalty0 368--387, 2010.

\bibitem[Ding and Lu(2017)]{ding2017principal}
Peng Ding and Jiannan Lu.
\newblock Principal stratification analysis using principal scores.
\newblock \emph{Journal of the Royal Statistical Society: Series B (Statistical
  Methodology)}, 79\penalty0 (3):\penalty0 757--777, 2017.

\bibitem[Dua and Graff(2017)]{DuaUCI}
Dheeru Dua and Casey Graff.
\newblock {UCI} machine learning repository, 2017.
\newblock URL \url{http://archive.ics.uci.edu/ml}.

\bibitem[Efron and Tibshirani(1986)]{efron1986bootstrap}
B.~Efron and R.~Tibshirani.
\newblock Bootstrap methods for standard errors, confidence intervals, and
  other measures of statistical accuracy.
\newblock \emph{Statistical Science}, 1:\penalty0 54–75, 1986.

\bibitem[Gruber et~al.(2019)Gruber, Lefebvre, Schuster, and Piché]{acic2019}
Susan Gruber, Geneviève Lefebvre, Tibor Schuster, and Alexandre Piché.
\newblock Atlantic causal inference conference data challenge, 2019.

\bibitem[Hand(2007)]{hand2007principles}
David~J Hand.
\newblock Principles of data mining.
\newblock \emph{Drug Safety}, 30\penalty0 (7):\penalty0 621--622, 2007.

\bibitem[Hern{\'a}n et~al.(2004)Hern{\'a}n, Hern{\'a}ndez-D{\'\i}az, and
  Robins]{hernan2004structural}
Miguel~A Hern{\'a}n, Sonia Hern{\'a}ndez-D{\'\i}az, and James~M Robins.
\newblock A structural approach to selection bias.
\newblock \emph{Epidemiology}, pages 615--625, 2004.

\bibitem[Hirano et~al.(2000)Hirano, Imbens, Rubin, and
  Zhou]{hirano2000assessing}
Keisuke Hirano, Guido~W. Imbens, Donald~B. Rubin, and Xiao-Hua Zhou.
\newblock Assessing the effect of an influenza vaccine in an encouragement
  design.
\newblock \emph{Biostatistics}, 1\penalty0 (1):\penalty0 69--88, 2000.
\newblock ISSN 1465-4644.

\bibitem[Imbens(2004)]{imbens2004}
Guido Imbens.
\newblock Nonparametric estimation of average treatment effects under
  exogeneity: A review.
\newblock \emph{Review of Economics and Statistics}, 2004.

\bibitem[Imbens and Rubin(2015)]{imbens2015causal}
Guido~W Imbens and Donald~B Rubin.
\newblock \emph{Causal inference in statistics, social, and biomedical
  sciences}.
\newblock Cambridge University Press, 2015.

\bibitem[Jiang and Ding(2017)]{jiang2017directions}
Z.~Jiang and P.~Ding.
\newblock The directions of selection bias.
\newblock \emph{Statistics and Probability Letters}, 125:\penalty0 104--109,
  2017.

\bibitem[Kallus et~al.(2018)Kallus, Puli, and Shalit]{kallus2018removing}
Nathan Kallus, Aahlad~Manas Puli, and Uri Shalit.
\newblock Removing hidden confounding by experimental grounding.
\newblock \emph{Advances in Neural Information Processing Systems (NeurIPS)},
  31, 2018.

\bibitem[Kleinberg(2013)]{kleinberg2013causal}
Samantha Kleinberg.
\newblock Causal inference with rare events in large-scale time-series data.
\newblock In \emph{International Joint Conference on Artificial Intelligence
  (IJCAI)}. Citeseer, 2013.

\bibitem[Lee et~al.(2020)Lee, Correa, and Bareinboim]{lee2020general}
Sanghack Lee, Juan Correa, and Elias Bareinboim.
\newblock General transportability--synthesizing observations and experiments
  from heterogeneous domains.
\newblock In \emph{Proceedings of the AAAI Conference on Artificial
  Intelligence}, volume~34, pages 10210--10217, 2020.

\bibitem[Maathuis et~al.(2010)Maathuis, Colombo, Kalisch, and
  B{\"u}hlmann]{maathuis2010predicting}
Marloes~H Maathuis, Diego Colombo, Markus Kalisch, and Peter B{\"u}hlmann.
\newblock Predicting causal effects in large-scale systems from observational
  data.
\newblock \emph{Nature Methods}, 7\penalty0 (4):\penalty0 247--248, 2010.

\bibitem[Maslove and Leisman(2019)]{maslove2019causal}
David~M Maslove and Daniel~E Leisman.
\newblock Causal inference from observational data: New guidance from
  pulmonary, critical care, and sleep journals.
\newblock \emph{Critical Care Medicine}, 47\penalty0 (1):\penalty0 1--2, 2019.

\bibitem[McDonald et~al.(1992)McDonald, Hui, and Tierney]{mcdonald1992effects}
C.~McDonald, S.~Hui, and W.~Tierney.
\newblock Effects of computer reminders for influenza vaccination on morbidity
  during influenza epidemics.
\newblock \emph{M.D. Computing : Computers in Medical Practice}, 9:\penalty0
  304--312, 1992.

\bibitem[Nalatore et~al.(2007)Nalatore, Ding, and
  Rangarajan]{nalatore2007mitigating}
Hariharan Nalatore, Mingzhou Ding, and Govindan Rangarajan.
\newblock Mitigating the effects of measurement noise on granger causality.
\newblock \emph{Physical Review E}, 75\penalty0 (3):\penalty0 031123, 2007.

\bibitem[Neyman(1923)]{neyman1923applications}
Jerzy Neyman.
\newblock Sur les applications de la thar des probabilities aux experiences
  agaricales: Essay des principle. excerpts reprinted (1990) in english.
\newblock \emph{Statistical Science}, 5\penalty0 (463-472):\penalty0 4, 1923.

\bibitem[Owen(2013)]{owen2013}
Art~B. Owen.
\newblock \emph{Monte Carlo Theory, Methods and Examples}.
\newblock 2013.

\bibitem[Pearl(2009)]{pearl2009causality}
Judea Pearl.
\newblock \emph{Causality}.
\newblock Cambridge university press, 2009.

\bibitem[Prentice and Pyke(1979)]{prentice1979logistic}
Ross~L Prentice and Ronald Pyke.
\newblock Logistic disease incidence models and case-control studies.
\newblock \emph{Biometrika}, 66\penalty0 (3):\penalty0 403--411, 1979.

\bibitem[Robins(2001)]{robins2001data}
James~M Robins.
\newblock Data, design, and background knowledge in etiologic inference.
\newblock \emph{Epidemiology}, pages 313--320, 2001.

\bibitem[Robins et~al.(2000)Robins, Hernan, and Brumback]{robins2000marginal}
James~M Robins, Miguel~Angel Hernan, and Babette Brumback.
\newblock Marginal structural models and causal inference in epidemiology,
  2000.

\bibitem[Rosenbaum(2002)]{rosenbaum2002}
Paul~R Rosenbaum.
\newblock \emph{Observational Studies}.
\newblock Springer, 2 edition, 2002.

\bibitem[Rosenbaum and Rubin(1983)]{rosenbaum1983central}
Paul~R Rosenbaum and Donald~B Rubin.
\newblock The central role of the propensity score in observational studies for
  causal effects.
\newblock \emph{Biometrika}, 70\penalty0 (1):\penalty0 41--55, 1983.

\bibitem[Rosenfeld et~al.(2017)Rosenfeld, Mansour, and
  Yom-Tov]{rosenfeld2017predicting}
Nir Rosenfeld, Yishay Mansour, and Elad Yom-Tov.
\newblock Predicting counterfactuals from large historical data and small
  randomized trials.
\newblock In \emph{Proceedings of the 26th International Conference on World
  Wide Web Companion}, pages 602--609, 2017.

\bibitem[Rosenman et~al.(2018)Rosenman, Owen, Baiocchi, and
  Banack]{rosenman2018propensity}
Evan Rosenman, Art~B. Owen, Michael Baiocchi, and Hailey Banack.
\newblock Propensity score methods for merging observational and experimental
  datasets, 2018.

\bibitem[Rosenman et~al.(2020)Rosenman, Basse, Owen, and
  Baiocchi]{rosenman2020combining}
Evan Rosenman, Guillaume Basse, Art Owen, and Michael Baiocchi.
\newblock Combining observational and experimental datasets using shrinkage
  estimators, 2020.

\bibitem[Rothman et~al.(2008)Rothman, Greenland, and Lash]{rothman2008modern}
Kenneth~J Rothman, Sander Greenland, and Timothy~L Lash.
\newblock \emph{Modern epidemiology}.
\newblock Lippincott Williams \& Wilkins, 2008.

\bibitem[Rubin(1974)]{rubin1974estimating}
Donald~B Rubin.
\newblock Estimating causal effects of treatments in randomized and
  nonrandomized studies.
\newblock \emph{Journal of Educational Psychology}, 66\penalty0 (5):\penalty0
  688, 1974.

\bibitem[Rubin(2006)]{rubin2006}
Donald~B Rubin.
\newblock \emph{Matched Sampling for Causal Effects}.
\newblock Cambridge University Press, 2006.

\bibitem[Schnabel et~al.(2016)Schnabel, Swaminathan, Singh, Chandak, and
  Joachims]{schnabel2016recommendations}
Tobias Schnabel, Adith Swaminathan, Ashudeep Singh, Navin Chandak, and Thorsten
  Joachims.
\newblock Recommendations as treatments: Debiasing learning and evaluation.
\newblock In \emph{International Conference on Machine Learning (ICML)},
  volume~48 of \emph{Proceedings of Machine Learning Research}, pages
  1670--1679, New York, New York, USA, 20--22 Jun 2016. PMLR.

\bibitem[Shao and Tu(2012)]{shao2012bootstrap}
J.~Shao and D.~Tu.
\newblock \emph{The Jackknife and Bootstrap}.
\newblock Springer, 2012.

\bibitem[Shiffrin(2016)]{shiffrin2016drawing}
Richard~M Shiffrin.
\newblock Drawing causal inference from big data.
\newblock \emph{National Academy of Sciences}, 113\penalty0 (27):\penalty0
  7308--7309, 2016.

\bibitem[Stuart et~al.(2013)Stuart, DuGoff, Abrams, Salkever, and
  Steinwachs]{stuart2013estimating}
Elizabeth~A Stuart, Eva DuGoff, Michael Abrams, David Salkever, and Donald
  Steinwachs.
\newblock Estimating causal effects in observational studies using electronic
  health data: challenges and (some) solutions.
\newblock \emph{Journal for Electronic Health Data and Methods}, 1\penalty0
  (3), 2013.

\bibitem[Tan(2006)]{tan2006distributional}
Zhiqiang Tan.
\newblock A distributional approach for causal inference using propensity
  scores.
\newblock \emph{Journal of the American Statistical Association}, 101\penalty0
  (476):\penalty0 1619--1637, 2006.

\bibitem[Wachinger et~al.(2019)Wachinger, Becker, Rieckmann, and
  P{\"o}lsterl]{wachinger2019quantifying}
Christian Wachinger, Benjamin~Gutierrez Becker, Anna Rieckmann, and Sebastian
  P{\"o}lsterl.
\newblock Quantifying confounding bias in neuroimaging datasets with causal
  inference.
\newblock In \emph{International Conference on Medical Image Computing and
  Computer-Assisted Intervention}, pages 484--492. Springer, 2019.

\bibitem[Wang et~al.(2016)Wang, Bendersky, Metzler, and
  Najork]{xuanhui2016learningtorank}
Xuanhui Wang, Michael Bendersky, Donald Metzler, and Marc Najork.
\newblock Learning to rank with selection bias in personal search.
\newblock In \emph{International ACM SIGIR Conference on Research and
  Development in Information Retrieval}, SIGIR '16, page 115–124, New York,
  NY, USA, 2016. Association for Computing Machinery.

\bibitem[Wang et~al.(2020)Wang, Liang, Charlin, and Blei]{wang2020causal}
Yixin Wang, Dawen Liang, Laurent Charlin, and David~M. Blei.
\newblock Causal inference for recommender systems.
\newblock In \emph{ACM Conference on Recommender Systems (RecSys)}, RecSys '20,
  page 426–431, New York, NY, USA, 2020. Association for Computing Machinery.

\bibitem[Yang and Ding(2020)]{yang2020combining}
Shu Yang and Peng Ding.
\newblock Combining multiple observational data sources to estimate causal
  effects.
\newblock \emph{Journal of the American Statistical Association}, 115\penalty0
  (531):\penalty0 1540--1554, 2020.

\bibitem[Zhang(2008)]{zhang2008completeness}
Jiji Zhang.
\newblock On the completeness of orientation rules for causal discovery in the
  presence of latent confounders and selection bias.
\newblock \emph{Artificial Intelligence}, 172\penalty0 (16-17):\penalty0
  1873--1896, 2008.

\bibitem[Zhang(2009)]{zhang2009estimatingoddsratio}
Zhiwei Zhang.
\newblock Estimating a marginal causal odds ratio subject to confounding.
\newblock \emph{Communications in Statistics—Theory and Methods},
  38:\penalty0 309--321, 02 2009.
\newblock \doi{10.1080/03610920802200076}.

\end{thebibliography}
